\documentclass[nohyperref]{article}

\usepackage{microtype}
\usepackage{graphicx}
\usepackage{subfigure}
\usepackage{booktabs} 
\usepackage{hyperref}

\usepackage[accepted]{icml2023}

\usepackage{amsmath}
\usepackage{amssymb}
\usepackage{mathtools}
\usepackage{amsthm}
\usepackage{amsfonts}      
\usepackage{nicefrac}      
\usepackage{microtype} 
\usepackage{bbm}
\usepackage{lipsum}
\usepackage{enumitem}
\usepackage{calc}
\usepackage{multirow}
\usepackage{tablefootnote}

\usepackage[capitalize,noabbrev]{cleveref}
\theoremstyle{plain}
\newtheorem{theorem}{Theorem}
\newtheorem{assumption}{Assumption}

\newtheorem*{theorem*}{Theorem}
\newtheorem*{lemma*}{Lemma}

\def\cS{{\mathcal{S}}}

\def\cA{{\mathcal{A}}}

\def\cZ{{\mathcal{Z}}}

\def\EE{{\mathbb{E}}}

\def\RR{{\mathbb{R}}}

\newcommand{\indep}{\perp \!\!\! \perp}

\newtheoremstyle{questionstyle}
  {\topsep}   % ABOVESPACE
  {0}         % BELOWSPACE
  {\itshape}  % BODYFONT
  {0pt}       % INDENT (empty value is the same as 0pt)
  {\bfseries} % HEADFONT
  {.}         % HEADPUNCT
  {5pt plus 1pt minus 1pt} % HEADSPACE
  {}          % CUSTOM-HEAD-SPEC
\theoremstyle{questionstyle}\newtheorem{question}{Question}

\usepackage[textsize=tiny]{todonotes}

\icmltitlerunning{Behavior Contrastive Learning for Unsupervised Skill Discovery}

\begin{document}

\twocolumn[
\icmltitle{Behavior Contrastive Learning for Unsupervised Skill Discovery
% in Reinforcement Learning
}



% \icmlsetsymbol{corr}{*}

\begin{icmlauthorlist}
\icmlauthor{Rushuai Yang}{shlab,hit}
\icmlauthor{Chenjia Bai}{shlab}
\icmlauthor{Hongyi Guo}{nw}
\icmlauthor{Siyuan Li}{hit}
\icmlauthor{Bin Zhao}{shlab,nwpu}
\icmlauthor{Zhen Wang}{nwpu}
\icmlauthor{Peng Liu}{hit}
\icmlauthor{Xuelong Li}{shlab,nwpu}
\end{icmlauthorlist}

\icmlaffiliation{shlab}{Shanghai Artificial Intelligence Laboratory, China}
\icmlaffiliation{hit}{Harbin Institute of Technology, China}
\icmlaffiliation{nw}{Northwestern University, USA}
\icmlaffiliation{nwpu}{Northwestern Polytechnical University, China}

\icmlcorrespondingauthor{Chenjia Bai}{baichenjia@pjlab.org.cn}
% \icmlcorrespondingauthor{Xuelong Li}{li@nwpu.edu.cn}

% \footnote{The work was done during the internship of Rushuang Yang at Shanghai AI Laboratory.}

\icmlkeywords{Machine Learning, ICML}

\vskip 0.3in
]




\printAffiliationsAndNotice

% leave blank if no need to mention equal contribution

\begin{abstract}
In reinforcement learning, unsupervised skill discovery aims to learn diverse skills without extrinsic rewards. Previous methods discover skills by maximizing the mutual information (MI) between states and skills. However, such an MI objective tends to learn simple and static skills and may hinder exploration. In this paper, we propose a novel unsupervised skill discovery method through contrastive learning among behaviors, which makes the agent produce similar behaviors for the same skill and diverse behaviors for different skills. Under mild assumptions, our objective maximizes the MI between different behaviors based on the same skill, which serves as an upper bound of the previous MI objective. Meanwhile, our method implicitly increases the state entropy to obtain better state coverage. We evaluate our method on challenging mazes and continuous control tasks. The results show that our method generates diverse and far-reaching skills, and also obtains competitive performance in downstream tasks compared to the state-of-the-art methods.
\end{abstract}

\section{Introduction}

Reinforcement Learning (RL) \cite{RLBook-2018} shows promising performance in a variety of challenging tasks, including game playing \cite{atari,Go}, quadrupedal locomotion \cite{quad-science-2020,quad-science-2022}, and robotic manipulation \cite{MT-OPT,daydreamer}. In most tasks, the policy is trained by optimizing a specific reward function with RL algorithms. One limitation of such a framework is that the learned policy is restricted to the training task and cannot generalize to other downstream tasks \cite{Procgen}. In contrast, humans can learn diverse and meaningful skills by exploring the environment without specific rewards, and then use these discovered skills to solve complex downstream tasks. Inspired by the human intelligence, unsupervised RL \cite{URLB} has recently become an important tool to address the generalization problem through skill discovery. 

Existing methods perform unsupervised skill discovery by maximizing the MI between states and skills. Specifically, the agent uses the discriminability and diversity of skill behaviors as an intrinsic reward, and then trains a skill-conditional policy by maximizing such a reward \cite{vic}. The discovered skills can serve as primitives to solve downstream tasks. Although these methods have shown promising results in discovering useful skills, a key limitation is that they often learn simple and static skills that lead to poor state coverage, which has also been observed in recent works \cite{EDL-2020,recurrent-2022}. For instance, in locomotion tasks, the MI objective tends to learn static `posing' skills rather than dynamic skills. We find it is mainly caused by the much smaller skill label space compared to the state space, which makes the MI objective easily achieve its maximum even with static skills. Specifically, the states visited by different skills can only have slight differences for distinguishing skills but not necessarily learn semantically meaningful or far-reaching skills. Increasing skill dimensions \cite{cic} or enforcing exploration \cite{liu2021aps,EDL-2020,LSD-2022} partially addresses this problem, while they require additional estimator or training techniques. 

In this study, we propose a novel method for unsupervised skill discovery, named \underline{Be}havior \underline{C}ontrastive \underline{L}earning (BeCL). We consider skill discovery from a multi-view perspective, where different trajectories condition on the same skill are different views. Specifically, we use the MI between different states generated by the same skill as an intrinsic rewards and train a policy to maximize it. Intuitively, BeCL encourages the agent to perform similarly condition on the same skill, and performs diversely among different skills. Under the redundancy condition, our MI objective serves as an upper bound of the previous MI objective, which means that BeCL also learns discriminating skills implicitly. Moreover, our method implicitly increases the state visitation entropy to encourage a better coverage of the environment, which prevents the agent from learning static skills. To tackle the MI estimation problem in high-dimensional control tasks, we adopt a contrastive learning algorithm to estimate the MI in BeCL by sampling positive and negative states from different skills. 

The contribution can be summarized as follows. (\romannumeral1) We propose a novel MI objective that performs unsupervised skill discovery and entropy-driven exploration simultaneously. (\romannumeral2) We propose BeCL as a practical contrastive method to approximate our objective in high-dimensional tasks. (\romannumeral3)~We discuss the connection and difference between BeCL and previous methods from an information-theoretical perspective. (\romannumeral4) We evaluate our method on maze tasks and Unsupervised RL Benchmark (URLB) \cite{URLB}. The result shows that BeCL learns diverse and far-reaching skills, and also demonstrates competitive performance in downstream tasks compared to the state-of-the-art methods. The open-sourced code is available at \url{https://github.com/Rooshy-yang/BeCL}.

\section{Preliminaries}

We consider a Markov Decision Process (MDP) defined as $(\cS, \cA, P, r, \gamma, \rho_0)$, where $\cS$ is the state space, $\cA$ is the action space, $P(s'|s,a)$ is the transition function, $\gamma$ is the discount factor, and $\rho_0: \cS\rightarrow [0,1]$ is the initial state distribution. For unsupervised skill discovery, the agent interacts with the environment without specific reward functions and $r$ denotes some intrinsic rewards. We denote the skill space by $\cZ$ and the skill $z\sim \cZ$ can be discrete or continuous vector. In each timestep $t$, 
the agent takes an action $a_t\sim\pi(a|s_t,z)$ by following the skill-conditional policy with a specific skill $z$. We denote the normalized probability that a policy $\pi$ encounters state $s$ as $\rho_\pi(s) \triangleq (1-\gamma)\sum_{t=0}^\infty \gamma^t \mathrm{P}_{t}^\pi(s)$.
% an agent observes the current state $s_t$, and then obtains the reward function $r_t$ and observe the next-state $s_{t+1}$. 

During the unsupervised training stage, the policy $\pi(a|s,z)$ is learned by maximizing the discounted cumulative intrinsic reward as $\sum_{t=0}^{T-1} \gamma^t r_t$. After training, we adapt the training skill to downstream tasks that have extrinsic reward functions $\{r^{\rm task}\}$. As suggested by URLB \cite{URLB}, we can choose a skill vector $z^\star$ and initialize the policy of the downstream task as $\pi(a|s,z^\star)$. Then we finetune the policy for a small number of interactions by optimizing a task-specific reward function $r^{\rm task}$ and measure the adaptation performance. Other metrics like data diversity and zero-shot transfer can also evaluate the quality of skills while they are less common than the adaptation efficiency.

In the following, we denote the information measure $I(\cdot;\cdot)$ as MI and $H(\cdot)$ as Shannon entropy. Previous skill discovery algorithms sample states from $\pi(a|s,z)$ and then maximize the MI objective $I(S;Z)$ using variational approximators, where $S$ and $Z$ denote random variables. The estimation of $I(S;Z)$ is used as an intrinsic reward to learn the policy $\pi(a|s,z)$. For example, DIAYN \cite{diayn} use a discriminator-based MI estimator as \begin{equation}\nonumber
I(S;Z)=H(Z)-H(Z|S)\geq H(Z)+\EE_{p(s,z)}[\log q_\phi(z|s)],
\end{equation}
where $q_\phi(z|s)$ is a discriminator of skills. With this discriminator, the intrinsic reward for skill discovery is set to $r=\log q_\phi(z|s) - \log p(z)$. Other variants also decompose $I(S;Z)$ to other forms or replacing $S$ by other transition information, while the initial MI objectives are similar. 

\section{The BeCL Method}

In this section, we first illustrate the shortcomings of existing skill discovery methods from an information-theoretical perspective. Then we propose a new MI objective for BeCL and give a practical approximation of the objective via contrastive learning. Finally, we give analyses to show the advantage of the BeCL algorithm in state coverage. 

\begin{figure}[t]
\begin{center}
\centerline{
\includegraphics[width=0.9\columnwidth]{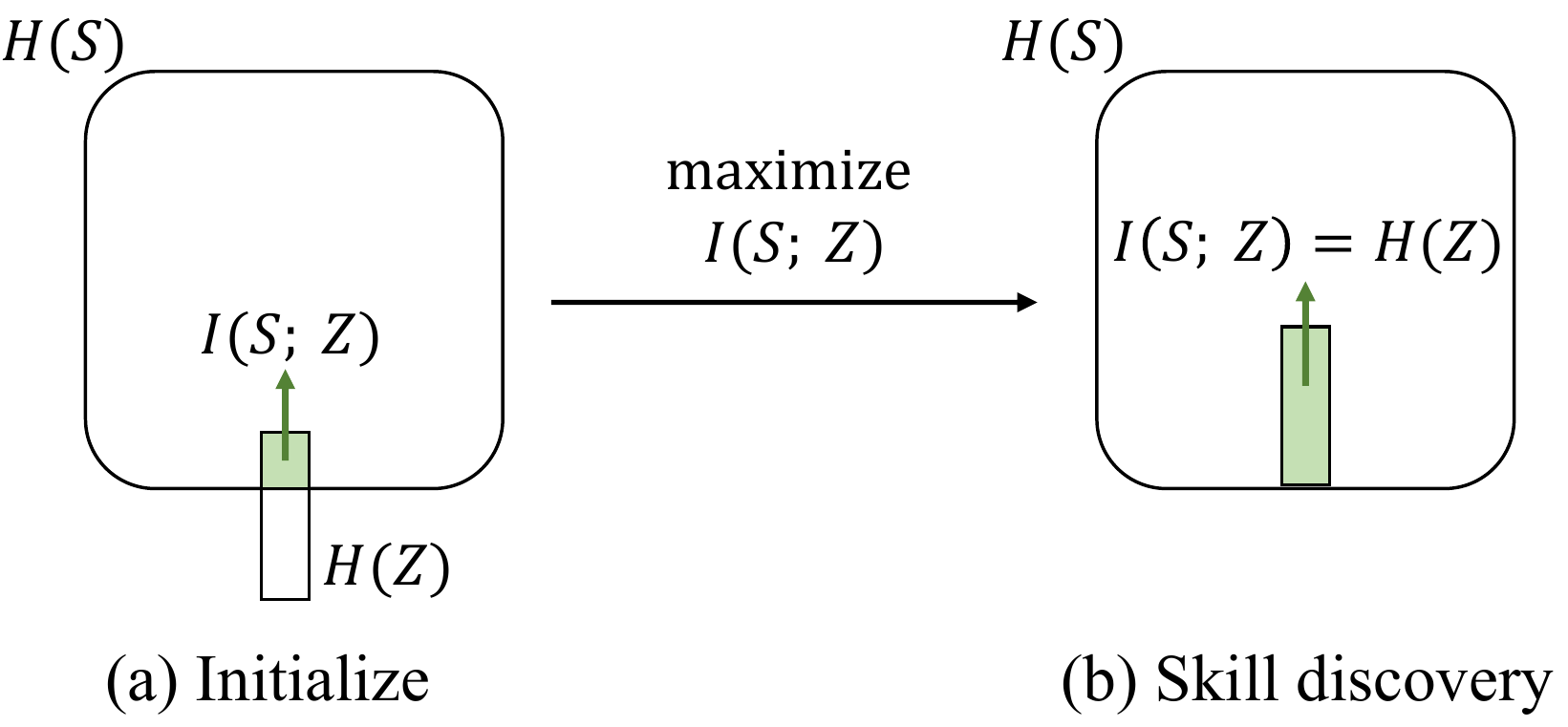}}
\vspace{-0.5em}
\caption{Information diagrams for maximizing $I(S;Z)$. (a)~The information shared between the states and skills (i.e., $I(S;Z)$) is less than $H(Z)$. (b) By maximizing the variational bound, the MI objective is maximized and we have $I(S;Z)=H(Z)$.}
\label{fig:mi-of-sz}
\vspace{-2em}
\end{center}
\end{figure}

\subsection{Limitation of Previous MI Objective}
\label{sec: Limitation_of_Previous_MI_Objective}

In skill discovery, a skill refers to a kind of behavior that can be represented by a set of semantically meaningful trajectories labeled by a fixed latent variable $z \sim \cZ $, indicating that one skill should map to a set of states. Formally, we have the following assumption.
\begin{assumption}
\label{assum-1}
% In skill discovery, 
The skill space $Z$ is smaller than the state visitation space, i.e., $H(Z) < H(S)$ with $s\sim \rho_\pi(s)$.
\end{assumption}
Empirically, Assumption~\ref{assum-1} holds even with a random policy since RL tasks often have high-dimensional state space, while the skill is a discrete or continuous vector with low dimensions. In contrast, if the skill space is equal or even larger than the state space, we have many-to-one mapping from skill to state, which makes the skills indistinguishable. Meanwhile, since the skill space is often set as a discrete or uniform distribution with fixed entropy, $H(Z)$ is a constant and is irrelevant to the optimization objective of $I(S;Z)=H(Z)-H(Z|S)$.

\begin{figure*}[t]
\begin{center}
\centerline{
\includegraphics[width=2.\columnwidth]{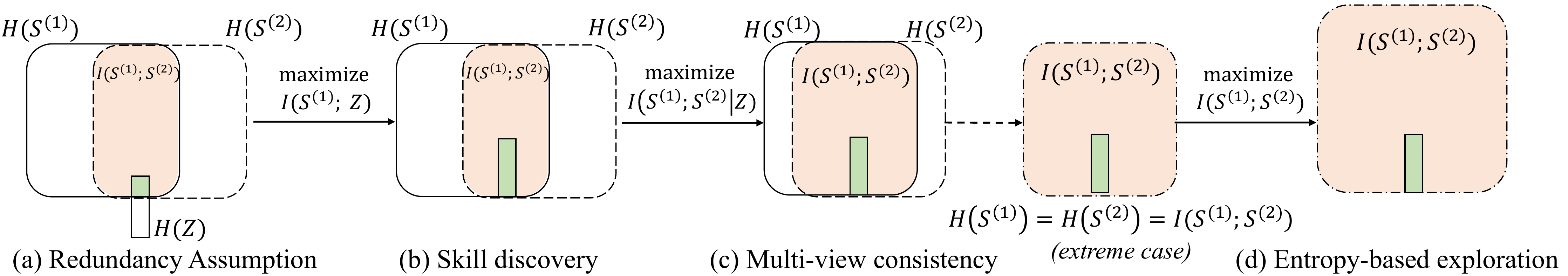}}
\caption{Information diagrams of the learning process in BeCL. (a) We show the redundancy assumption of skill between different views. (b) BeCL performs skill discovery by maximizing $I(S;Z)$ and (c) learns multi-view consistency by maximizing $I(S^{(1)};S^{(2)}|Z)$. (d)~Considering the maximum empowerment and perfect multi-view consistency, our objective directly maximizes the state coverage.}
\label{fig:mi-of-BeCL}
\end{center}
\end{figure*}

We illustrate the information diagrams for maximizing $I(S;Z)$ in Figure~\ref{fig:mi-of-sz}. (\romannumeral1) As shown in Figure~\ref{fig:mi-of-sz}(a), since the states are generated by the skill-conditional policy and the policy is randomly initialized, the states and skills usually share a small amount of information and $I(S;Z) \geq 0$. (\romannumeral2) As we learn a discriminator and train the policy to maximize the discriminability of skills, the MI term $I(S;Z)$ increases and approaches $H(Z)$. As in Figure~\ref{fig:mi-of-sz}(b), when the MI objective is fully optimized, the information of $I(S;Z)$ is contained in $H(Z)$. Under Assumption~\ref{assum-1}, we have 
\begin{equation}\label{eq:limit-sz}
\max I(S;Z) = \max H(Z) - H(Z|S) = H(Z),
\end{equation}
where $\min H(Z|S)=0$ due to the non-negativity of entropy, and $\max H(Z)=H(Z)$ since $H(Z)$ is fixed.

The major limitation of above optimization process is such an MI objective can hinder better state coverage and learn static skills. Specifically, the optimization objective is irrelevant to the state visitation entropy under Assumption~\ref{assum-1} and we have $\max I(S;Z)=H(Z)$ with different $H(S)$. Considering there are two policies $\pi_1$ and $\pi_2$, where $\pi_2$ has better state coverage. Then the relationship between the corresponding state entropy measured by the states visited by the two policies is $H_{\rho_{\pi_1}}(S)< H_{\rho_{\pi_2}}(S)$. Nevertheless, for both policies, the maximum values of the $I(S;Z)$ objective are equal to $H(Z)$, which is irrelevant to their state entropy as shown in Eq.~\eqref{eq:limit-sz}. Since there is no gradient to encourage the agent to update from $\pi_1$ to $\pi_2$, the MI objective can easily obtain its maximum, which discourages the agent from learning far-reaching skills. We remark that such a problem has also been mentioned in previous studies \cite{EDL-2020,LSD-2022}, where they observe the agent visits marginally different states to learn distinguishable skills (e.g., with different static postures). Nevertheless, we highlight that we give an information-theoretical perspective and propose an alternative MI-objective to solve this problem. 

\subsection{The BeCL Objective}

In this section, we introduce the proposed BeCL objective and show how our objective addresses above limitations. We consider a multi-view setting where each skill $z_i$ generates two trajectories $\tau^{(1)}$ and $\tau^{(2)}$ based on the skill-conditional policy. The trajectories $\tau^{(1)}$ and $\tau^{(2)}$ usually have differences since (\romannumeral1) the skill has limited empowerment to behaviors in the early stage of training and (\romannumeral2) the transition function and policy are often stochastic. Based on $\tau^{(1)}$ and $\tau^{(2)}$, we sample two states by following 
\begin{equation}\nonumber
s^{(1)}\sim \tau^{(1)},\quad s^{(2)}\sim \tau^{(2)},
\end{equation}
where $\tau^{(1)},\tau^{(2)}\sim \pi(a|s,z_i), \forall i$. We remark that $s^{(1)}$ and $s^{(2)}$ are generated with the same skill and denote the corresponding random variables by $S^{(1)}$ and $S^{(2)}$, respectively. 

Similar to multi-view representation that assumes each view shares the same task-relevant information \cite{MIB-2020,MIB-2022}, we assume that $S^{(1)}$ and $S^{(2)}$ share the same skill-relevant information since they are generated on the same skill vector $z$. Formally, we give the following mutual redundancy assumption in skill discovery. 
\begin{assumption}[Redundancy]
\label{assump-2}
The states $S^{(1)}$ and $S^{(2)}$ are mutually redundant to the skill-relevant information, i.e., $S^{(1)} \indep Z | S^{(2)}$ or equivalently $I(S^{(1)},Z|S^{(2)})=0$.
\end{assumption}

\begin{figure*}[t]
\begin{center}
\centerline{
\includegraphics[width=2.0\columnwidth]{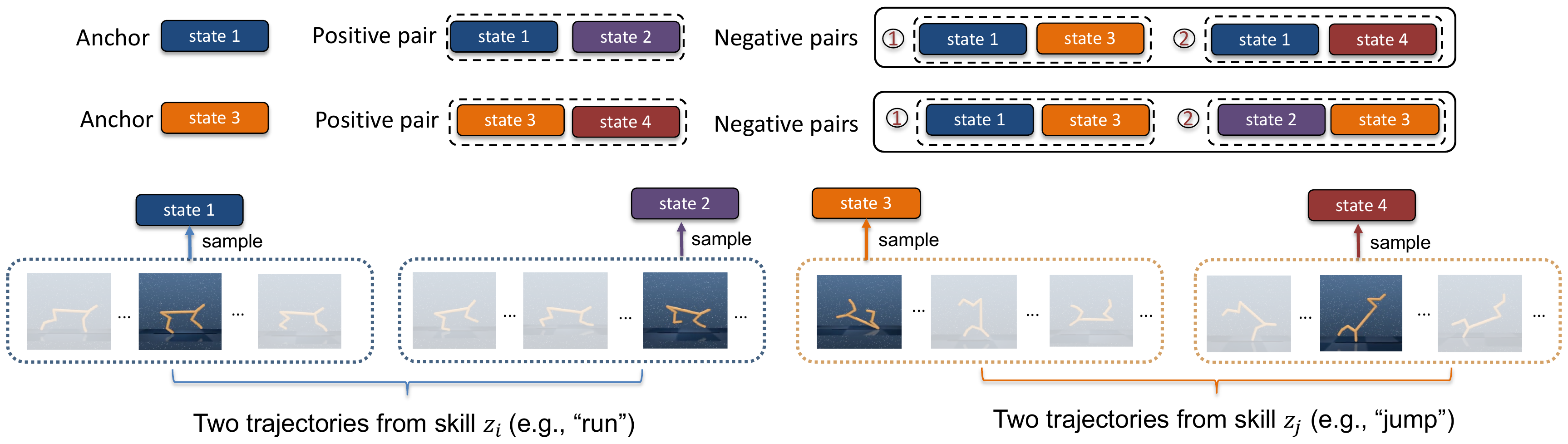}}
\caption{An illustration of the process of contrastive estimation in BeCL. We take two states of different trajectories generated by the same skill as positive samples. In addition, we find that sampling positive pairs from the same trajectory that are far from each other is also well-performed. Meanwhile, we take states of different trajectories generated by different skills as negative samples. Consider \textit{state 1} as example, we use \textit{(state 1, state 2)} as positive pair, and then use \textit{(state 1, state 3)} and \textit{(state 1, state 4)} as negative pairs. The contrastive loss and intrinsic reward are computed in Eq.~\eqref{eq:nce-tau} and Eq.~\eqref{eq:reward}, respectively.}
\label{fig:cl-diagram}
\end{center}
\vspace{-1em}
\end{figure*}

Under Assumption \ref{assump-2}, we have 
\begin{equation}
\label{eq:mi-multivariate}
I(S^{(1)};Z)=I(S^{(2)};Z)=I(S^{(1)};S^{(2)};Z).
\end{equation}
where the last term is multivariate MI. The proof is given in Appendix~\ref{app-theory}. We remark that Assumption \ref{assump-2} does not indicate that the skill has good discriminability (i.e., $I(S,Z)$ is large) initially, but only assuming that the skill information shared between different views (i.e., $S^{(1)}$, $S^{(2)}$) are the same since they are generated by the same skill, as shown in Eq.~\eqref{eq:mi-multivariate}. In most cases, the redundancy assumption holds unless the policy network does not extract any skill information. An illustration of this assumption is given in Figure~\ref{fig:mi-of-BeCL}(a).

Based on the redundancy assumption, we propose a novel MI objective for our method, as 
\begin{equation}
I_{\rm BeCL}=I(S^{(1)};S^{(2)}).
\end{equation}
To analyze the objective, we decompose $I_{\rm BeCL}$ as
\begin{equation}\label{eq:mi-becl}
\begin{aligned}
&I_{\rm BeCL}=I(S^{(1)};S^{(2)})\\
&=I(S^{(1)};S^{(2)};Z)+I(S^{(1)};S^{(2)}|Z)\\
&=\underbrace{\nicefrac{1}{2}\big[I(S^{(1)};Z)+I(S^{(2)};Z)\big]}_{(\romannumeral 1)~\rm skill\:\:discovery}\:+\!\!\!\underbrace{I(S^{(1)};S^{(2)}|Z)}_{(\romannumeral 2)~\rm multi-view\:\: consistency},
\end{aligned}
\end{equation}
where the last equation follows Eq.~\eqref{eq:mi-multivariate}. We analyze the optimization process of $I_{\rm BeCL}$ as follows. 

(1) \textit{Skill discovery}. Maximizing term $(\romannumeral 1)$ in Eq.~\eqref{eq:mi-becl} is equivalent to maximizing $I(S;Z)$ in previous skill discovery methods, where $s\sim S$ can be sampled from $S^{(1)}$ or $S^{(2)}$. As a result, our objective serves as an upper bound of $I(S;Z)$ in previous MI objective \cite{vic}. As shown in Figure~\ref{fig:mi-of-BeCL}(b), this objective expands the information shared between $Z$ and $(S^{(1)},S^{(2)})$. In practice, since $H(Z)$ is much smaller than $H(S)$, this term is relatively easy to optimize to achieve its maximum value (i.e., $H(Z)$).

(2) \textit{Multi-view consistency}. The empowerment of skill to the resulting states is maximized when term $(\romannumeral 1)$ in Eq.~\eqref{eq:mi-becl} achieve $H(Z)$. Then the differences between two views (i.e., $S^{(1)}$ and $S^{(2)}$) come from the randomness of the environment and the policy. As shown in Figure~\ref{fig:mi-of-BeCL}(c), through maximizing the second term in Eq.~\eqref{eq:mi-becl} (i.e., $I(S^{(1)},S^{(2)}|Z)$), the skills will learn to make the explored policy less sensitive to environmental randomness and also drive the policy becomes deterministic, which leads to better multi-view consistency of states based on the same skill. We remark that such an effect is different from previous methods that maximize the policy entropy to keep the policy stochastic \cite{diayn}. Nevertheless, we remark that a deterministic skill-conditional policy can also have a well state coverage by generating far-reaching trajectories. The multi-view consistency makes trajectories sampled with the same skill explore the nearly areas and have better alignment.
 % which encourage the agent to learn meaningful skills
% and explore the nearly areas, which makes the skills less sensitive to the environment noises. 

(3) \textit{Entropy-based Exploration.} In Figure~\ref{fig:mi-of-BeCL}(c), we show an extreme case that $S^{(1)}$ and $S^{(2)}$ are completely consistent by optimizing the multi-view consistency in our objective. In this case, we have the relationship $I(S^{(1)};S^{(2)})=H(S^{(1)})=H(S^{(2)})$ holds, thus maximizing our objective directly maximizes the entropy of explored states. We highlight that such a property addresses the limitation of the previous MI objective. As illustrated in Figure~\ref{fig:mi-of-BeCL}(d), the state coverage increases as we encourage the policy to increase the state visitation entropy. Although such a complete consistency case is unachievable in practice, we will show that BeCL indeed maximizes the state entropy through contrastive estimation of our MI objective in the following.

\subsection{Behavior Contrastive Learning}

To estimate the MI objective in high-dimensional state space, a tractable variational estimator based on neural networks is required \cite{on-MI-bound}. In BeCL, we adopt contrastive learning \cite{cl-2018} to approximate the objective, as 
\begin{equation}\label{eq:nce}
\begin{aligned}
L_{\rm BeCL1}&= \EE_{z_i,\{z_j\}\sim p(z), (s_i^{(1)},s_i^{(2)})\sim p(\cdot|z_i)} \EE_{s_j\sim p(\cdot|z_j),\forall z_j\neq z_i} \\
\qquad &\left[ -\log \:\: \frac{h(s^{(1)}_i,s^{(2)}_i)}{\sum_{s_j \in S^{-}\bigcup s^{(2)}_i}h(s^{(1)}_i,s_j)}\right],
\end{aligned}
\end{equation}
where $z_i$ and $z_j$ are skills sampled from a discrete skill distribution $p(z)$. A positive pair $[s^{(1)}_i$, $s^{(2)}_i]$ contains two states sampled from trajectories based on the same skill-conditional policy $\pi(a|s,z_i)$. In contrast, a negative pair $[s^{(1)}_i, s_j]$ is constructed by sampling $s_i$ and $s_j$ based on different skills $z_i$ and $z_j$. In Eq.~\eqref{eq:nce}, we use a negative sample set $S^{-}$ to represent the states sampled from $\{z_j\}$, where $z_j\neq z_i$. Since $s_i$ and $s_j$ are sampled independently by different skills, we expect their similarity to be much smaller than that of positive pairs. The score function is an exponential similarity measurement between states, as
\begin{equation}
h(s_i,s_j)=\exp\big(f_\phi(s_i)^{\top} f_\phi(s_j)\big),
\end{equation}
where $f_\phi(\cdot)$ is an encoder network with normalization to make $\|f_\phi(\cdot)\|=1$. The function $h(\cdot,\cdot)$ aims to assign high scores for positive pairs and low scores for negative pairs. 

An illustration of the contrastive learning process is given in Figure~\ref{fig:cl-diagram}. We sample $m$ skills $\{z_i\}_{i\in[m]}$ and generate two trajectories for each skill when interacting with the environment. 
Based on the trajectories $\{\tau^{(1)}_i,\tau^{(2)}_i\}_{i\in[m]}$, we sample a state from each trajectory, which gives a positive pair as $\{s^{(1)}_i,s^{(2)}_i\}_{i\in[m]}$. Then, we construct the negative pairs by sampling states come from trajectories based on different skills. For each anchor state, we use the corresponding positive state and $2(m-1)$ negative states to calculate the contrastive loss.

\begin{theorem}
\label{thm:nec-upper}
The relationship between our MI objective in Eq.~\eqref{eq:mi-multivariate} and the contrastive loss defined in Eq.~\eqref{eq:nce} is 
\begin{equation}
I_{\rm BeCL} = I(S^{(1)};S^{(2)}) \geq \log N - L_{\rm BeCL1},
\end{equation}
where $N=2m-1$ and $m$ is the number of sampled skills.
\end{theorem}

The proof is given in Appendix~\ref{app-theory}. Since $N$ is a constant, minimizing the contrastive loss will maximize our MI objective $I(S^{(1)};S^{(2)})$. As discussed in Figure~\ref{fig:mi-of-BeCL}, maximizing $I(S^{(1)};S^{(2)})$ can perform skill discovery, improve the multi-view consistency, and improve the state coverage.

Comparing to other RL algorithms that use data augmentation \cite{srinivas2020curl}, dynamics consistency \cite{infomax-2020}, or temporal information \cite{eysenbach2020goal-conditioned} for contrastive learning, we provide a novel way by using skills to construct positive and negative pairs, which extracts skill-relevant information for unsupervised RL.

\subsection{Qualitative Analysis}

In practice, we follow SimCLR \cite{simclr} by using a small temperature $\kappa<1$ in our contrastive objective to control the strength of penalties. Specifically, we define
\begin{equation}\label{eq:nce-tau}
\begin{aligned}
&L_{\rm BeCL2}=\EE_{\: i,j\in[m],s^{(1)}_i,s^{(2)}_i,s_j}\\
&=\left[-\log\frac{\exp\big(f(s^{(1)}_i)^{\top}f(s^{(2)}_i)/\kappa \big)}{\sum_{s_j \in S^{-}\bigcup s_i^{(2)}}\exp \big(f(s_j)^{\top}f(s^{(1)}_i)/\kappa \big)}\right].
\end{aligned}
\end{equation}
There are two effects in optimizing Eq.~\eqref{eq:nce-tau} in skill discovery. (\romannumeral1) The positive pairs in the numerator become similar in the feature space, achieving better multi-view alignment for trajectories based on the same skill. For implementation, $s_i$ or $s_j$ in Eq.~\eqref{eq:nce-tau} can also be replaced by other information (e.g., sub-trajectory) to capture the long-term consistency. (\romannumeral2) The negative samples in the denominator have repulsive force to each other, which will push the states in one skill away from states in other skills. Finally, the states of different skills will roughly be uniformly distributed in the state space, which helps the agent improve the state coverage. 

Formally, since the state features lie on a hypersphere, i.e., $\{f(s)\in\mathbb{R}^d:\|f(s)\|=1\}$, we adopt the von Mises-Fisher (vMF) distribution as a spherical density function to perform kernel density estimation (KDE) \cite{kde-2019,wang2020alignment_and_uniformity}. The following theorem establishes the relationship between the contrastive objective $L_{\rm BeCL2}$ and the state entropy estimation \cite{entropy-1997}. 

\begin{theorem}
\label{thm:entropy}
With sufficient negative samples, minimizing $L_{\rm BeCL2}$ can maximize the state entropy, as
\begin{equation}\label{eq:thm2}
\!\!\lim_{N\rightarrow \infty}\!L_{\rm BeCL2} \!=\! 
-\frac{1}{\kappa} \EE_{s_i}[f(s^{(1)}_i)^\top f(s^{(2)}_i)] - \hat{H}\big(f(s)\big) + \log C,
\end{equation}
where $\hat{H}(\cdot)$ is a resubstitution entropy estimator through the von Mises-Fisher (vMF) kernel density estimation, and $\log C$ is a normalization constant. 
\end{theorem}

A detailed proof is given in Appendix~\ref{app-theory}. In Theorem~\ref{thm:entropy}, the vMF kernel has a concentration parameter $\kappa^{-1}$, which controls how peaky of the distribution is around its referenced feature in entropy estimation. In practice, we set $\kappa=0.5$ to obtain a reasonable result. According to Theorem~\ref{thm:entropy}, the first term of Eq.~\eqref{eq:thm2} is related to skill discovery methods \cite{diayn,liu2021aps} by improving the multi-view alignment to increase the empowerment of skills. The second term is similar to data-based unsupervised RL methods \cite{apt,cic} as they explicitly measure the state entropy through a particle entropy estimator, while we maximize the state entropy implicitly via contrastive learning. 

For unsupervised RL training, we set the intrinsic reward to 
\begin{equation}\label{eq:reward}
r(s_i^{(1)}) \!:=\! \EE_{s_i^{(2)},s_j}
\!\!\left[\frac{\exp\big(f(s^{(1)}_i)^{\top}f(s^{(2)}_i)/\kappa \big)}{\sum_{s_j \sim S^{-}\bigcup s_i^{(2)}}\exp \big(f(s_j)^{\top}f(s^{(1)}_i)/\kappa \big)}\right],
\end{equation}
where we estimate the reward by sampling $s_i^{(2)}$ and $S^{-}$ from a training batch that contains $2m$ trajectories. Then we use the intrinsic reward for policy training. The algorithmic description of our method is given in Appendix~\ref{app-alg}.

\section{Related Works}

\paragraph{Unsupervised Skill Discovery}
Unsupervised skill discovery allows agents to learn discriminable behavior by maximizing the MI between states and skills \cite{vic, florensa2017stochastic, diayn, dads, Baumli_Warde-Farley_Hansen_Mnih_2021}. However, many works \cite{URLB, cic, disdain} have shown that skill learning through variational MI maximization provides a poor coverage of state space, which may affect its applicability to downstream tasks with complex environments. Some methods consider restricting the observation space of skill learning to x-y Cartesian coordinates to increase in traveled distances (or variations) in the coordinate space (called the x-y prior) \cite{LSD-2022, MUSIC}. However, these methods introduce strong assumptions in learning skills and are mainly narrow in navigation tasks and a few other tasks with coordinate information. In addition, other methods also propose auxiliary exploration mechanisms or training techniques \cite{disdain, pmlr-v139-bagaria21a, NEURIPS2019_251c5ffd, EDL-2020, recurrent-2022} to tackle the problem of state coverage.
% while they often constitute several modules and increase the computation cost.
%For example, EDL \cite{EDL-2020} separate the learning process by first optimize $H(S)$ via an exceptional entropy-based exploration policy, and followed by skill learning and discovery on the basis of collected samples. DISDIAYN \cite{disdain} introduces N discriminators with different parameters to alleviate the Epistemic uncertainty of the sample from the discriminator and ReST \cite{recurrent-2022} trains different skills one after another recurrently to avoid the interference of different skills when visiting same states. % 
In contrast, our approach learns diverse skills without limiting the observation space, and also implicitly increases the state entropy to encourages exploration without auxiliary losses. 

% We empirically show that BeCL learns useful skills to improve the adaptation ability of downstream tasks without coordinate information. 

% not sure whether write or not, it should be put into other section ?
% In addition, to apply our skill discovery method to high dimensional environments. the first problem that we need to tackle is the means of evaluations. In fact, RL community does not have general and objective metrics evaluating the quality of behavior derived from different skills, especially in some high-dimensional and less-semantic environments, even though some papers argue that (QD evaluate the frequency of gaits), and (DADS, DIAYN evaluate with 2\^|z| ??, and aforementioned diversity of position ). Because our method mainly focuses on unsupervised learning methods, We follow the means of evaluation about competence-based methods from URLB and assert that standard that learns semantic skills if the agent can learn such behavior, so we just concentrate on URLB, which provides pretrain stage only depended on intrinsic reward without expert demonstrations using supervised learning to predict the primitives taken by experts.% 

\paragraph{Unsupervised RL} Unsupervised RL algorithms focus on training a general policy for fast adaptation to various downstream tasks. Unsupervised RL mainly contains two stages: pretrain and finetune. 
% In the pretrain stage, the agent interacts with the environment to learn a policy only with intrinsic rewards. In the finetune stage, the agent finetunes the pretrained policy with extrinsic rewards given the specific downstream tasks. 
The core of unsupervised RL is to design an intrinsic reward in the pretrain stage.
% There are many methods learns self-supervised representations of environment models and specific behavior advanced in the pretrain stage enable sample-efficient few-shot adaptation when facing specific tasks. 
% 
URLB \cite{URLB} broadly divides existing algorithms into three categories. The data-based approaches encourage the agent to collect novel states in pretraining through maximization of state entropy \cite{apt, proto, cic}; the knowledge-based approaches enable agents to learn behavior based on the output of some model prediction (i.e. curiosity, surprise, uncertainty, etc.) \cite{pathak2017curiosity, pathak2019disagreement, burda2018rnd,bai-1,bai-2,bai-3}; and the competence-based approaches maximize the agent empowerment over environment from the perspective of information theory, which means the agent is trained to discover what can be done in an environment while learning how to achieve it \cite{lee2019smm,diayn, liu2021aps}. Our method can be regarded as a novel competence-based approach as we can make each skill learn potential behavior that is useful for downstream tasks and also improve the state coverage.

\paragraph{Contrastive Learning}
Contrastive learning is a representation learning framework in deep learning \cite{he2020moco,SwAV,grill2020bootstrap, CLIP}. The main idea is to define positive and negative pairs to learn useful representations. Contrastive learning has also been used in RL as auxiliary tasks to improve the sample efficiency. Positive and negative pairs in RL can be constructed following state enhancement \cite{srinivas2020curl}, temporal consistency \cite{sermanet2018time-contrastive}, dynamic-relevant transition \cite{infomax-2020,bai2021dynamic,qiu2022Contrastive_ucb}, return estimation \cite{liu2021Return-basedCL}, or goal information \cite{eysenbach2020goal-conditioned}. Unlike these methods, we use skills to divide the positive and negative pairs. Recently proposed CIC \cite{cic} performs contrastive learning to approximate $I(S;Z)$ and uses the learned representation for entropy estimation. In contrast, we propose a different contrastive objective through a multi-view perspective and use the objective as an intrinsic reward. Our method is related to the alignment and uniformity properties of contrastive learning mentioned in \citet{wang2020alignment_and_uniformity}. Our contrastive objective also learns uniformly distributed state representations to improve the state coverage.

\section{Experiments}
\vspace{-0.5em}

In this section, we first provide qualitative analysis for the behaviors of different skills learned in BeCL and other competence-based methods in a 2D continuous maze and challenging continuous control tasks from the DeepMind Control Suite (DMC) \cite{dmc}. We then compare the adaptation efficiency of the learned skills in downstream tasks of URLB \cite{URLB}, where previous competence-based methods shown to produce relatively weak performance. We finally conduct several ablation studies on skill finetuning, skill dimensions, and temperature.

\subsection{Continuous 2D Maze}

We start from a continuous 2D maze to illustrate the skills learned in BeCL. We adopt the environment of \citet{EDL-2020}, where the agent observes its current position ($\cS \in \RR^2$) and takes an action ($\cA \in \RR^2$) to control its velocity and direction. The agent will be blocked if it collides with the wall. We make comparisons with three typical skill optimization objectives from DIAYN \cite{diayn}, DADS \cite{dads}, and CIC \cite{cic}. Specifically, DIAYN maximizes the reverse form of MI as $I(S;Z) = H(Z) - H(Z | S)$, DADS maximizes the forward form of MI as $I(S;Z) = H(S) - H(S | Z)$, and CIC is a data-based method that maximizes the state entropy $H(S)$ with a particle estimator. 
% We implement the environments based on the codebase in \citet{EDL-2020}. 
For a fair comparison, all methods sample skills from a 10-dimensional discrete distribution and follow the same training procedure. The differences between methods are the formulation of intrinsic rewards and representations.

\begin{question}
Can BeCL balance skill empowerment and state coverage?
\end{question}

\begin{figure}[t]
    \centering
    \resizebox{\linewidth}{!}{\includegraphics{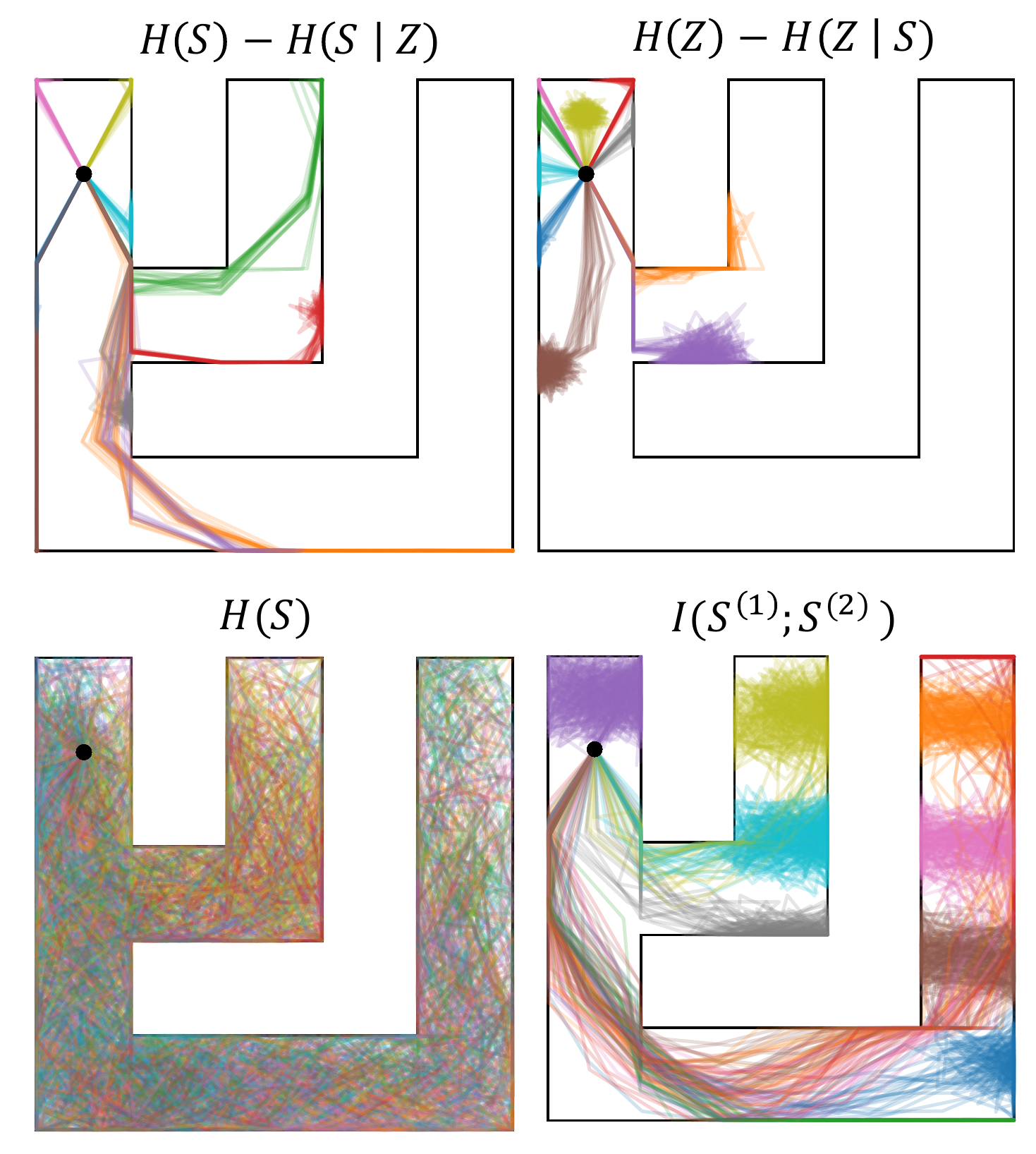}}
    \vspace{-1em}
    \caption{A illustration of different skill discovery objectives. The initial state is denoted by a black dot and the color of the trajectories denote different skill upon which it was conditioned. We generate 20 trajectories for each skill. \textit{Top:} the reverse \cite{diayn} and forward \cite{dads} forms for optimizing $I(S; Z)$ can discover discriminable skills but fail to reach the right side of the maze. \textit{Bottom left:} Maximizing the state entropy allows skills to cover the entire space of maze, while different skills cannot be distinguished from each other \cite{cic}. \textit{Bottom right:} BeCL learns skills that are distributed in different areas with well alignment in the same skill, indicating that BeCL can balance state coverage and empowerment in skill discovery. }\label{fig:comparison_all_maze}
\end{figure} 

\begin{figure*}[t]
\begin{center}
\centerline{
\includegraphics[width=2.\columnwidth]{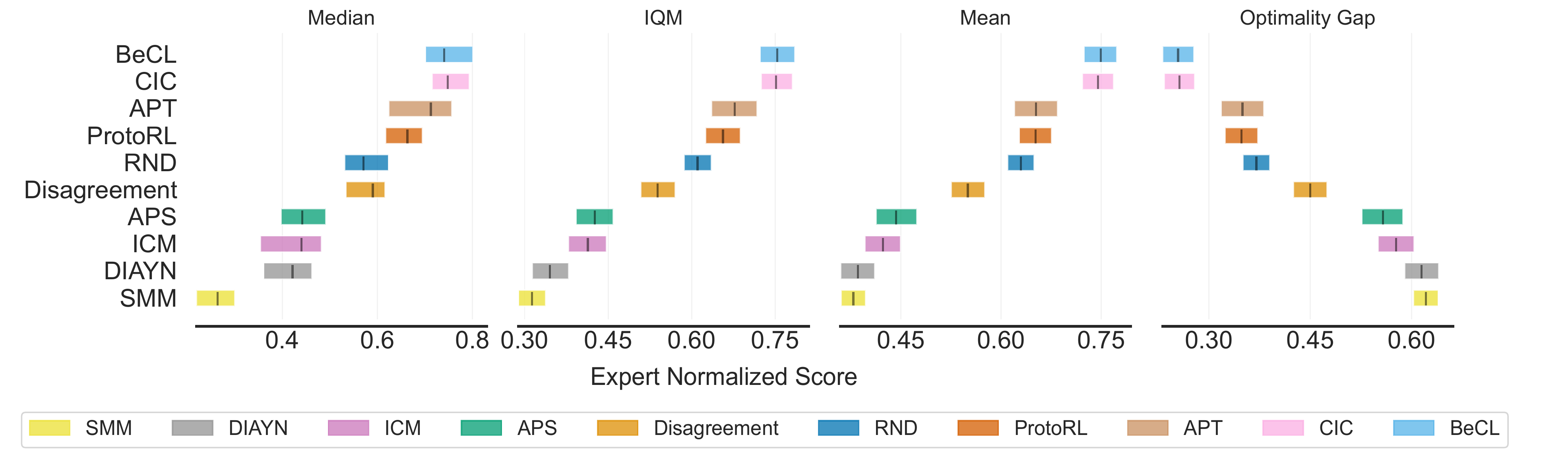}}
\vspace{-0.5em}
\caption{The aggregate statistics indicate the adaptation performance of different unsupervised RL methods in 12 downstream tasks. We run all baselines for 10 seeds and report the aggregated normalized score \citep{agarwal2021IQM} after 100K steps of finetune. BeCL obtains the highest Interquratile Mean (IQM) score of 75.38\% and the lowest Optimality Gap (OG) score of 25.44\%.}
\label{fig:iqm_metrics}
\end{center}
\vspace{-1em}
\end{figure*}

We visualize the trajectories generated by different skills in maze and compare BeCL to other skill discovery methods in Figure~\ref{fig:comparison_all_maze}. We find DIAYN and DADS produce discriminable skills while the trajectories are not far-reaching. In contrast, CIC can produce skills with the best state coverage, while the trajectories of different skills are mixed and lack of diversity. Since CIC maximizes the state entropy via a \textit{k-nearest-neighbor} estimator and maps states from the same skill to similar representations, the closest $k$-th neighbors can also include states conditioned on the same skill. Maximizing the state entropy will push these states away from each other and leads to an increase in $H(S | Z)$. In contrast, BeCL maps the states from the same skill into similar features to encourage better empowerment; meanwhile, it also pushes states in one skill away from states in other skills to obtain diverse skills and better state coverage. We provide further numerical analysis of mutual information and entropy estimation between the methods in the maze, as shown in Figure~\ref{fig:MI and Entropy Estimation} in Appendix \ref{app:Additional Implementation in Maze}.

\subsection{URLB Environments}

We evaluate BeCL in DMC tasks from URLB benchmark \cite{URLB}, which needs to discover more complicated skills to achieve the desired behavior. URLB consists of three different domains, including Walker, Quadruped, and Jaco Arm. Each domain contains four downstream tasks. Specifically, Walker is a biped constrained to a 2D vertical plane (i.e. $\cS \in \RR^{24}, \cA \in \RR^{6}$), which has four different locomotion tasks including \textit{(Walker, Stand)}, \textit{(Walker, Run)}, \textit{(Walker, Flip)} and \textit{(Walker, Walk)}; Quadruped is a quadruped with four downstream tasks including \textit{(Quadruped, Stand)}, \textit{(Quadruped, Run)}, \textit{(Quadruped, Jump)} and \textit{(Quadruped, Walk)}, while it is more challenging due to the higher-dimensional states and actions spaces (i.e. $\cS \in \RR^{78}, \cA \in \RR^{16}$) and complex dynamics; Jaco Arm is a 6-DoF robotic arm with a three-finger gripper (i.e. $\cS \in \RR^{55}, \cA \in \RR^{9}$) and its downstream tasks aim to reach and manipulate a movable diamond with different positions. We illustrate more details about the tasks in Figure~\ref{fig:intro_env_urlb} of Appendix~\ref{app:Additional Experiments in URLB}

\begin{question}
What skills do BeCL learn in DMC?
\end{question}

% Previous competence-based methods cannot learn diverse and dynamic skills in the DMC \cite{dmc} as they do in the OpenAI gym \cite{openAIgym}. This is probably due to the difference in episodic terminal settings (e.g., whether it resets when the agent loses balance), which is discussed in \citet{URLB,cic}. 

We provide a qualitative analysis of the behavior of skills in more complex control tasks. We compare BeCL to other competence-based algorithms including CIC and DIAYN in Walker domain. As shown in Figure~\ref{fig:comparison_dmc_behavior} in Appendix \ref{app-vis-skill-dmc}, with the same pretraining steps, BeCL produces dynamic and non-trivial behavior during a finite episode. In contrast, solely maximizing state entropy like CIC leads to trivial and dynamic behavior since it encourages the agent to collect unusual states, such as visiting `handstands' state by constantly trying to wiggle the agent's body for larger reward. In addition, DIAYN learns static `posing' skills since the static skills can also optimize the MI objective $I(S;Z)$, as we discuss in Section~\ref{sec: Limitation_of_Previous_MI_Objective}. 

\begin{question}
How does the adaptation efficiency of BeCL compared to other unsupervised RL algorithms?
\end{question}

\paragraph{Baselines.} We compare BeCL with other baselines in the URLB benchmark, including knowledge-based, data-based, and competence-based algorithms. Knowledge-based methods include ICM \cite{pathak2017curiosity}, Disagreement \cite{pathak2019disagreement}, and RND \cite{burda2018rnd}; data-based methods include APT \cite{apt}, ProtoRL \cite{proto}, and CIC \cite{cic}; and competence-based methods include SMM \cite{lee2019smm}, DIAYN \cite{diayn}, and APS \cite{liu2021aps}. The main difference between baselines are the design of intrinsic reward and representation. We summarize the implementation details of the baselines in Appendix~\ref{app:description_of_baselines_in_URLB}. We follow the hyper-parameters and the implementation recommended by URLB and CIC.

\paragraph{Evaluation.} To perform a fair comparison, we follow standard pretraining and finetuning procedures as suggested in URLB \cite{URLB}. We pretrain each algorithm for 2M steps with only intrinsic rewards for each domain, and then finetune the policy for 100K steps in different downstream tasks with extrinsic reward. We use DDPG \cite{ddpg} as the basic RL algorithm and train each method for 10 seeds, resulting in 1200 runs in total (i.e., $10$ algorithms $\times$ $10$ seeds $\times$ $3$ domains $\times$ $4$ tasks). 

Following \textit{reliable} \cite{agarwal2021IQM}, we adopt interquartile mean (IQM) and optimality gap (OG) metrics aggregated with stratified bootstrap sampling as our main evaluation metrics across all runs. IQM discards the bottom and top 25\% of the runs and calculates the mean score of the remaining 50\% runs. OG evaluates the amount by which the algorithm fails to meet a minimum score of desired target. The expert score is obtained by running DDPG with 2M steps in the corresponding tasks and we adopt the expert scores from \citet{cic}. We normalize each score with the expert score and the statistical results are shown in Figure~\ref{fig:iqm_metrics}. In the IQM metric, BeCL achieves competitive performance with CIC (75.38\% and 75.18\%, respectively) and outperforms the next best skill discovery algorithm (i.e., APS) by 38.2\%. In OG metric, BeCL achieves the closest performance to expert performance (around 25.44\%) and the CIC score achieves approximately 25.75\%.

\subsection{Ablation Study}

% We conduct several ablation studies of the BeCL as follows.

\begin{question}
Whether different skills have different adaptation efficiency on downstream tasks?
\end{question}

We evaluate the adaptation efficiency of the learned skills in pretraining the \textit{Quadruped} agent. We compare the normalized reward after finetuning by initializing the policy with different skills vector. We show the downstream performance in \textit{(Quadruped, stand)} and \textit{(Quadruped, run)} tasks in Figure~\ref{fig:ablation_finetune_skills} of Appendix~\ref{app:Additional Experiments in URLB}. We find that the performance of different skills does not always revolve around statistical averages and some skills have relatively weak adaptation ability, which indicates that a skill-chosen process would be desired before finetuning. For example, CIC \cite{cic} chooses skills with grid sweep in the first 4K finetuning steps. In contrast, BeCL randomly samples skills in the finetuning stage to report the average performance of skills, which provides a more comprehensive evaluation of the learned skills.

\begin{question}
How does the skill dimensions affect unsupervised skill discovery?
\end{question}

We explore the impact of the skill dimensions with a discrete skill space. We train DIAYN \cite{diayn}, DADS \cite{dads} and BeCL with different numbers of skill in a tree-like maze. As shown in Figure~\ref{fig:comparison_skill_dim_maze} of Appendix~\ref{app:Additional Implementation in Maze}, with the skill dimension increases, DIAYN and DADS still optimize MI in a narrow area of the maze and cannot go far away. In contrast, BeCL skills gradually cover the entire maze when the skill dimension increases, which coincides to Theorem~\ref{thm:entropy} that using more skills increases the number of negative samples and provides a better entropy estimator. We further study the impact of the skill dimension on adaptation efficiency in DMC, as shown in Figure~\ref{fig:ablation_study_urlb} of Appendix~\ref{app:Additional Experiments in URLB}. The results show that increasing the skill dimension can benefit the adaptation performance in hard downstream tasks (e.g. \textit{(Walker, Run)} and \textit{(Walker, Flip)}), while it cannot improve the performance in relatively easy tasks approaching expert scores (e.g. \textit{(Walker, Stand)}).

\begin{figure*}[h!]
\begin{center}
\centerline{
\includegraphics[width=2\columnwidth]{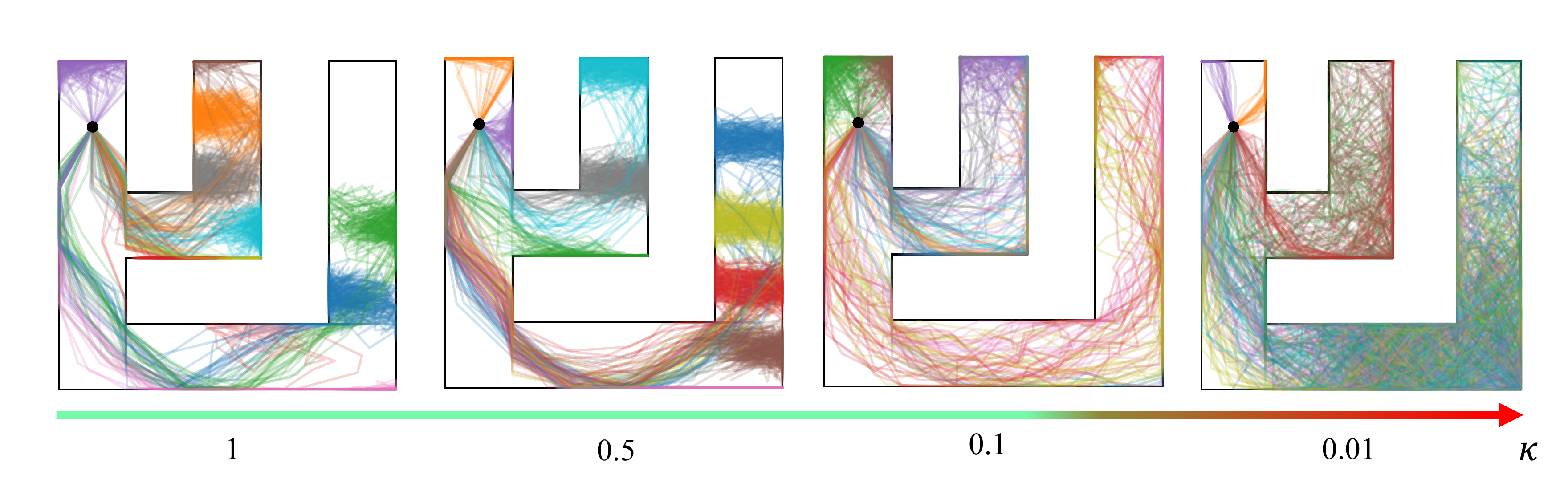}}
\vspace{-1em}
\caption{The impact of temperature $\kappa$ in the behavior of BeCL. Lower temperatures encourage more skills to cover the right side of the maze (e.g. temperature from 1 to 0.5). But as the temperature dropped further (e.g. temperature from 0.1 to 0.01), some skills tend to produce more dispersed trajectories. This makes some skills indistinguishable from each other in a confined maze.}
\label{fig:ablation_temp_maze}
\end{center}
\vspace{-1em}
\end{figure*}

\begin{question}
How does the temperature $\kappa$ in the contrastive objective affect empowerment and state coverage?
\end{question}

To analyze the effect of temperature $\kappa$ on the behavior of BeCL skills, we evaluate several values of $\kappa$ in the maze task. The result is shown in Figure~\ref{fig:ablation_temp_maze}. We find that using a smaller temperature value encourages the skills to explore the entire maze space (i.e., decreases from 1 to 0.5). However, as the temperature decreases further (i.e., decreases from 0.1 to 0.01), trajectories from different skills tend to be more uniformly distributed, which somewhat weakens the discriminability of skills. We remark that such a scenario resembles previous analysis of temperature \cite{wang2021understanding,wang2020alignment_and_uniformity}, where $\kappa$ is considered to balance the uniformity-tolerance dilemma in unsupervised contrastive learning. Specifically, we hope that the skills can be distributed uniformly to be globally separated; meanwhile, we hope that the trajectories are locally clustered and are more tolerant to the states generated with the same skill. A suitable $\kappa$ may avoid the gathering of different skills in local areas, which helps to explore different regions and improves the state coverage. Meanwhile, too small $\kappa$ can be harmful to the alignment of positive samples, which weakens the empowerment of skills.

\section{Conclusion}

We propose BeCL with a novel MI objective for unsupervised skill discovery from a multi-view perspective. Theoretically, BeCL discovers skills and maximizes the state coverage simultaneously, which makes BeCL produce diverse behaviors in various skills. Empirical results show that BeCL learns diverse and far-reaching skills in mazes and performs well in downstream tasks of URLB. In the future, we will improve BeCL by selecting hard negative samples to obtain a tighter MI bound, and developing a meta-controller for better skill selection in the finetuning stage.

\section*{Acknowledgements}

This work was done during Rushuai Yang's internship at Shanghai Artificial Intelligence Laboratory. This work was supported by Shanghai Artificial Intelligence Laboratory and the National Science Fund for Distinguished Young Scholars under Grant 62025602. The authors thank Kang Xu, Haoran He, Qichen Zhao, Dong Wang and Zhigang Wang for the helpful discussions. We also thank anonymous reviewers, whose invaluable comments and suggestions have helped us to improve the paper.

\nocite{langley00}

\bibliography{main}
\bibliographystyle{icml2023}

%%%%%%%%%%%%%%%%%%%%%%%%%%%%%%%%%%%%%%%%%%%%%%%%%%%%%%%%%%%%%%%%%%%%%%%%%%%%%%%
%%%%%%%%%%%%%%%%%%%%%%%%%%%%%%%%%%%%%%%%%%%%%%%%%%%%%%%%%%%%%%%%%%%%%%%%%%%%%%%
% APPENDIX
%%%%%%%%%%%%%%%%%%%%%%%%%%%%%%%%%%%%%%%%%%%%%%%%%%%%%%%%%%%%%%%%%%%%%%%%%%%%%%%
%%%%%%%%%%%%%%%%%%%%%%%%%%%%%%%%%%%%%%%%%%%%%%%%%%%%%%%%%%%%%%%%%%%%%%%%%%%%%%%
\newpage
\appendix
\onecolumn
\icmltitle{Appendix}

\section{Theoretical Proof}\label{app-theory}

\subsection{Proof of the MI Decomposition}

Under Assumption~\ref{assump-2}, we recall that Eq.~\eqref{eq:mi-multivariate} is 
\begin{equation}\nonumber
I(S^{(1)};Z)=I(S^{(2)};Z)=I(S^{(1)};S^{(2)};Z).
\end{equation}

\begin{proof}
For random variables $X$, $Y$ and $Z$, the chain rule for multivariate MI is 
\begin{equation}
I(X;Y;Z)=I(Y;Z)-I(Y;Z|X).
\end{equation}

Thus, for $S^{(1)}$, $S^{(2)}$, and $Z$, we have the similar relationship as
\begin{equation}
I(S^{(1)};Z)=I(S^{(1)};S^{(2)};Z)+I(S^{(1)};Z|S^{(2)}).
\end{equation}
According to the redundancy assumption in Assumption~\ref{assump-2}, we have 
\begin{equation}
I(S^{(1)};Z\:\big|\:S^{(2)})=0,
\end{equation}
and then we have 
\begin{equation}
I(S^{(1)};Z)=I(S^{(1)};S^{(2)};Z).
\end{equation}
Following a similar proof, we have 
\begin{equation}
I(S^{(2)};Z)=I(S^{(1)};S^{(2)};Z),
\end{equation}
which conclude our proof.
\end{proof}

\subsection{Proof of Theorem~\ref{thm:nec-upper}}

\begin{theorem*}[Theorem~\ref{thm:nec-upper} restate]
The relationship between our MI objective in Eq.~\eqref{eq:mi-multivariate} and the contrastive loss defined in Eq.~\eqref{eq:nce} is 
\begin{equation}
I_{\rm BeCL} = I(S^{(1)};S^{(2)}) \geq \log N - L_{\rm BeCL1},
\end{equation}
where $N=2m-1$ and $m$ is the number of sampled skills.
\end{theorem*}

\begin{proof} 
We rewrite $L_{\rm BeCL1}$ defined in Eq.~\eqref{eq:nce} with $m$ discrete skills as
\begin{equation}\nonumber
L_{\rm BeCL1}= \EE_{i\in[m],s_i^{(1)},s_i^{(2)},s_j}\left[ -\log \:\: \frac{h(s^{(1)}_i,s^{(2)}_i)}{\sum_{s_j \in S^{-}\bigcup s^{(2)}_i}h(s^{(1)}_i,s_j)}\right],
\end{equation}

Following the definition of MI, we have
\begin{equation}
\begin{aligned}
I(S^{(1)};S^{(2)})-\log N &=
\EE_{i\in[m],s^{(1)}_i,s^{(2)}_i}\left[\log \frac{p(s^{(1)}_i s^{(2)}_i)}{p(s^{(1)}_i) p(s^{(2)}_i)}\right] - \log{N} \\
&= \EE_{i\in[m],s^{(1)}_i,s^{(2)}_i}\left[\log \frac{p(s^{(1)}_i|s^{(2)}_i)}{p(s^{(1)}_i)}\right] - \log{N} \\
&= \EE_{i\in[m],s^{(1)}_i,s^{(2)}_i}\left[\log \frac{1}{\frac{p(s^{(1)}_i)}{p(s^{(1)}_i|s^{(2)}_i)}N}\right] \\
&= \EE_{i\in[m],s^{(1)}_i,s^{(2)}_i}\left[\log \frac{1}{\frac{p(s^{(1)}_i)}{p(s^{(1)}_i|s^{(2)}_i)} + \frac{p(s^{(1)}_i)}{p(s^{(1)}_i|s^{(2)}_i)}(N-1)}\right] \\
\end{aligned}
\end{equation}

Under Assumption~\ref{assump-2}, since $S^{(1)}$ and $S^{(2)}$ share the same information about the skill, we have $p(s^{(1)}|s^{(2)})\geq p(s^{(1)})$. Then
\begin{align}
I(S^{(1)};S^{(2)})-\log N & \ge \EE_{i\in[m],s^{(1)}_i,s^{(2)}_i}\left[\log \frac{1}{1 + \frac{p(s^{(1)}_i)}{p(s^{(1)}_i|s^{(2)}_i)}(N-1)}\right] \\
&= \EE_{i\in[m],s^{(1)}_i,s^{(2)}_i}\left[-\log \left(1 + \frac{p(s^{(1)}_i)}{p(s^{(1)}_i|s^{(2)}_i)}(N-1)\right)\right]
\label{eq:app-nega}.
\end{align}

Considering we sample $s_j\in S^{-}$ that is independent to $s^{(1)}_i$ (i.e., $i\neq j$), we have $p(s^{(1)}_i|s_j)=p(s^{(1)}_i)$. Formally, we have 
\begin{equation}
\label{eq:app-sj}
\EE_{j\in [m] \setminus \{i\}, s_j\in S^{-}}\left[\frac{p(s^{(1)}_i|s_j)}{p(s^{(1)}_i)}\right]=1,
\end{equation}
where $S^{-}$ contains $N-1$ negative examples. Plugging \eqref{eq:app-sj} into \eqref{eq:app-nega}, we have
\begin{equation}
\begin{aligned}
I(S^{(1)};S^{(2)})-\log N  & \ge \EE_{i\in[m],s^{(1)}_i,s^{(2)}_i}\left[-\log \left(1 + \frac{p(s^{(1)}_i)}{p(s^{(1)}_i|s^{(2)}_i)}(N-1)\right)\right] \\
&= \EE_{i\in[m],s^{(1)}_i,s^{(2)}_i}\left[-\log \left(1 + \frac{p(s^{(1)}_i)}{p(s^{(1)}_i|s^{(2)}_i)}(N-1) \EE_{ S^{-}}\left[\frac{p(s^{(1)}_i|s_j)}{p(s^{(1)}_i)}\right]\right)\right]
\\
&= \EE_{i\in[m],s^{(1)}_i,s^{(2)}_i,s_j}\left[-\log \left(1 + \frac{p(s^{(1)}_i)}{p(s^{(1)}_i|s^{(2)}_i)}\sum_{s_j\in S^{-}}\frac{p(s^{(1)}_i|s_j)}{p(s^{(1)}_i)}\right)\right]\\
&= \EE_{i\in[m],s^{(1)}_i,s^{(2)}_i,s_j}\left[-\log\left(\frac{p(s^{(1)}_i|s^{(2)}_i)+p(s^{(1)}_i)\sum_{s_j\in S^{-}}\frac{p(s^{(1)}_i|s_j)}{p(s^{(1)}_i)}}{p(s^{(1)}_i|s^{(2)}_i)}\right)\right]\\
&= \EE_{i\in[m],s^{(1)}_i,s^{(2)}_i,s_j}\left[\log\left(\frac{\frac{p(s^{(1)}_i|s^{(2)}_i)}{p(s^{(1)}_i)}}{\frac{p(s^{(1)}_i|s^{(2)}_i)}{p(s^{(1)}_i)}+\sum_{s_j\in S^{-}}\frac{p(s^{(1)}_i|s_j)}{p(s^{(1)}_i)}}\right)\right]\\
&= \EE_{i\in[m],s^{(1)}_i,s^{(2)}_i,s_j} \left[ \log \left(\frac{h(s^{(1)}_i,s^{(2)}_i)}{h(s^{(1)}_i,s^{(2)}_i) + \sum_{s_j\sim S^{-}} h(s^{(1)}_i,s_j)}\right) \right],
\end{aligned}
\end{equation}
where $h(x_1, x_2)=p(x_1|x_2) / p(x_1)$ is the score function that preserves the mutual information between $x_1$ and $x_2$. In practice, we use a neural network to represent the score function.  

\end{proof}

\subsection{Proof of Theorem~\ref{thm:entropy}}

\begin{theorem*}[Restate of Theorem~\ref{thm:entropy}]
With sufficient negative samples, minimizing $L_{\rm BeCL2}$ can maximize the state entropy, as
\begin{equation}
\!\!\lim_{N\rightarrow \infty}\!L_{\rm BeCL2} \!=\! 
-\frac{1}{\kappa} \EE_{s_i}[f(s^{(1)}_i)^\top f(s^{(2)}_i)] - \hat{H}\big(f(s)\big) + \log C,
\end{equation}
where $\hat{H}(\cdot)$ is a resubstitution entropy estimator through the von Mises-Fisher (vMF) kernel density estimation, and $\log C$ is a normalization constant. 
\end{theorem*}

\begin{proof}
We rewrite the definition of our contrastive estimator as 
\begin{equation}
\begin{aligned}
\label{eq:nce-tau2}
L_{\rm BeCL2}&=\EE_{s_i,s_j}
\left[-\log\frac{\exp\big(f(s^{(1)}_i)^{\top}f(s^{(2)}_i)/\kappa \big)}{\exp\big(f(s^{(1)}_i)^{\top}f(s^{(2)}_i)/\kappa\big) + \sum_{s_j \in S^{-}}\exp \big(f(s_j)^{\top}f(s^{(1)}_i)/\kappa \big)}\right]\\
&=\EE_{s_i}\left[-\frac{1}{\kappa} \big(f(s^{(1)}_i)^{\top}f(s^{(2)}_i)\right] + \EE_{s_i,s_j}\left[\log \left( \exp\big(f(s^{(1)}_i)^{\top}f(s^{(2)}_i)/\kappa\big) + \sum_{s_j \in S^{-}}\exp \big(f(s_j)^{\top}f(s^{(1)}_i)/\kappa \big) \right)\right]
\end{aligned}
\end{equation}

In the following, we denote $M=N-1$ as the number of negative samples when using $s^{(1)}$ as the anchor state. Then we have
\begin{equation}
\begin{aligned}
&\lim_{M\rightarrow\infty} L_{\rm BeCL2} - \log M  \\
&= \EE_{s_i}\left[-\frac{1}{\kappa} \big(f(s^{(1)}_i)^{\top}f(s^{(2)}_i)\right] + \EE_{s_i} \lim_{M\rightarrow\infty} \left[\log \left( \frac{1}{M}\exp\big(f(s^{(1)}_i)^{\top}f(s^{(2)}_i)/\kappa\big) + \frac{1}{M}\sum_{s_j \in S^{-}}\exp \big(f(s_j)^{\top}f(s^{(1)}_i)/\kappa \big) \right)\right] \\
&= \EE_{s_i}\left[-\frac{1}{\kappa} \big(f(s^{(1)}_i)^{\top}f(s^{(2)}_i)\right] + \EE_{s_i} \lim_{M\rightarrow\infty} \left[\log \left( \frac{1}{M}\sum_{s_j \in S^{-}}\exp \bigg(f(s_j)^{\top}f(s^{(1)}_i)/\kappa \bigg) \right)\right],
\end{aligned}
\end{equation}
where the second equation holds by the strong law of large numbers (SLLN). 

When $M\rightarrow \infty$, the negative sample set $S^{-}$ contains sufficient states to represent the state visitation distribution. As a result, sampling $s_j\in S^{-}$ will be equivalent to sampling $s_j\sim \rho_{\pi}(s)$, where $\rho_{\pi}(s)$ is the state visitation measure of the current policy $\pi$. Then we have
\begin{equation}\label{eq:app-M}
\lim_{M\rightarrow\infty} L_{\rm BeCL2} = \EE_{s_i}\left[-\frac{1}{\kappa} f(s^{(1)}_i)^{\top}f(s^{(2)}_i)\right] +  \EE_{s_i}\Big[\log \EE_{s_j \sim \rho_{\pi}(s)} \big[\exp \big(f(s_j)^{\top}f(s^{(1)}_i)/\kappa \big) \big]\Big] + \log M,
\end{equation}
where we follow Continuous Mapping Theorem with the logarithmic function. 

As we normalize the output of the encoder network to make $\|f(\cdot)\|=1$, the features of states lie on a unit hypersphere
$$\mathbb{S}^{d-1}:\big\{f(s)\in \mathbb{R}^{d}: \|f(s)\|=1\big\}.$$
Kernel density estimation (KDE) is commonly used in the Euclidean setting. In our problem, we use a \textit{spherical kernel} $K$ as a spherical probability density function with a mean direction $\mu\in\mathbb{R}^{d}$ and a concentration parameter $u>0$. Specifically, we adopt the classical von Mises-Fisher (vMF) distribution defined over 
\begin{equation}\label{eq:app-vmf}
K_{\rm vMF}(x;\mu,u) = Z_{\rm vMF}(u)\cdot \exp(u\cdot\mu^{\top}x), \qquad Z_{\rm vMF}(u)=\frac{u^{d/2-1}}{(2\pi)^{d/2}I_{d/2-1}(u)},
\end{equation}
where $I_{\alpha}$ is a modified Bessel function of the first kind with order $\alpha$. Here, $u$ and $\mu\in\mathbb{R}^d$ are the parameters of vMF density with $u\geq 0$ and $\|\mu\|=1$. 

In the following, we denote $s_i=s^{(1)}_i$ for the second term in Eq.~\eqref{eq:app-M} as it only contains a single view that can be sampled from $\rho_\pi (s)$. In the following, we sample $s_i$ with a number of $N_{\rm s}$ to estimate this term. With a sufficiently large number of $N_{\rm s}$ and the vMF kernel defined in Eq.~\eqref{eq:app-vmf}, we have
\begin{equation}\label{eq:app-ent}
\begin{aligned}
\EE_{s_i}\bigg[\log  \EE_{s_j \sim \rho_{\pi}(s)} \Big[\exp \big(f(s_j)^{\top}f(s^{(1)}_i)/\kappa \big) \Big]\bigg] &= \frac{1}{N_{\rm s}} \sum_{i=1}^{N_{\rm s}} \log\left(\frac{1}{M} \sum_{j=1}^{M} \exp\big(f(s_j)^{\top}f(s_i)/\kappa \big) \right) \\
&= \frac{1}{N_{\rm s}} \sum_{i=1}^{N_{\rm s}} \log \hat{p}_{\rm vMF-KDE} \big(f(s_i)\big) + \frac{1}{N_{\rm s}} \sum_{i=1}^{N_{\rm s}} \log Z_{\rm vMF}\big(f(s_i)\big)\\
&=-\hat{H}\big(f(s)\big) + \log Z_{\rm vMF},
\end{aligned}
\end{equation}
where we denote $Z_{\rm vMF}=\prod_{i=1}^{N_{\rm s}}Z_{\rm vMF}(f(s_i))$ as the normalization constant for the vMF distribution. Here, $\hat{p}_{\rm vMF-KDE}$ is the vMF kernel density estimation with a concentration parameter of $u=\kappa^{-1}$. 

With the density estimation in the hypersphere, $\hat{H}$ is a resubstitution entropy estimator based on vMF. Inserting Eq.~\eqref{eq:app-ent} into Eq.~\eqref{eq:app-M} gives us
\begin{equation}
\lim_{M\rightarrow\infty} L_{\rm BeCL2} = \EE_{s_i}\left[-\frac{1}{\kappa} f(s^{(1)}_i)^{\top}f(s^{(2)}_i)\right] - \hat{H}(f(s)) + \log Z_{\rm vMF} + \log M,
\end{equation}
which concludes our proof by setting $C= M\cdot Z_{\rm vMF}$. 

\end{proof}
\clearpage

\section{Additional Experiments in Maze}
\label{app:Additional Implementation in Maze}
% The toy example environments are adapted from the open-source implementation by EDL \url{https://github.com/victorcampos7/edl}. where DIAYN and DADS have been implemented in there. We replicate CIC and BeCL to the maze and remain its hyperparameter and RL backbone same as DIAYN and DIAYN. The main difference between baselines is the optimization loss of network and the design of intrinsic reward.

\subsection{Effect of Skill Dimensions}
\begin{figure}[h]
\begin{center}
\centerline{
\includegraphics[width=0.8\columnwidth]{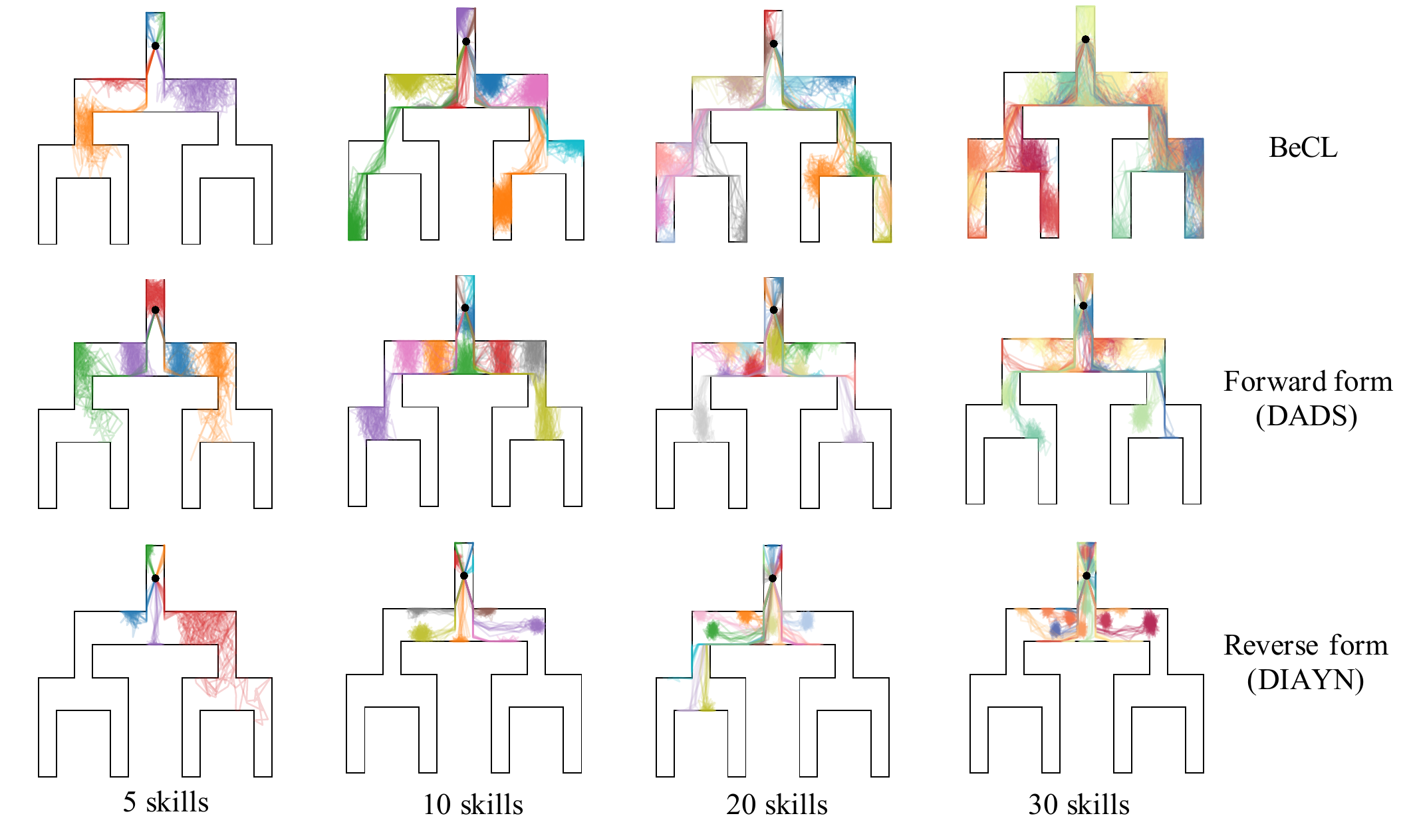}}
\vspace{-1em}
\caption{A illustration of the impact of different numbers of skills in tree-like Mazes. All methods are trained for 2500 episodes, and each episode allows 50 interactions with environment. As the skill dimension increases, skills of DIAYN and DADS still optimize MI in a narrow area of the maze and cannot reach deeper position. In contrast, BeCL skills gradually cover the whole state space by applying more skills to explore the maze and more negative samples to provide a better state-entropy estimator.}
\vspace{-1em}
\label{fig:comparison_skill_dim_maze}
\end{center}
\end{figure}

\subsection{The MI and Entropy Estimation of skill discovery methods in Maze}

\begin{figure}[h]
\begin{center}
\centerline{
\includegraphics[width=0.4\columnwidth]{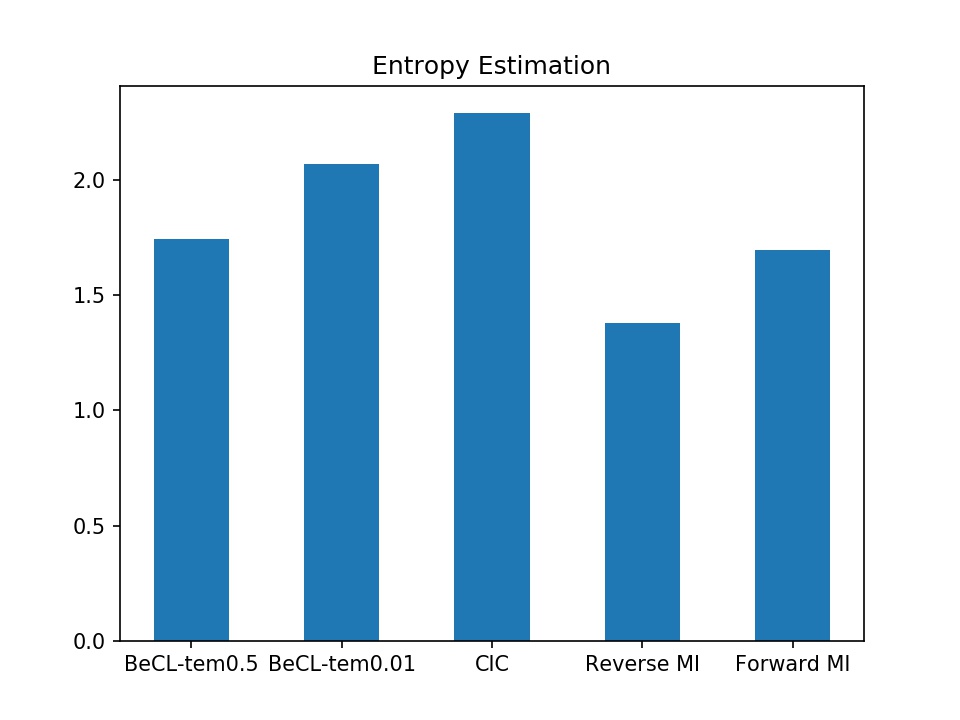}
\includegraphics[width=0.4\columnwidth]{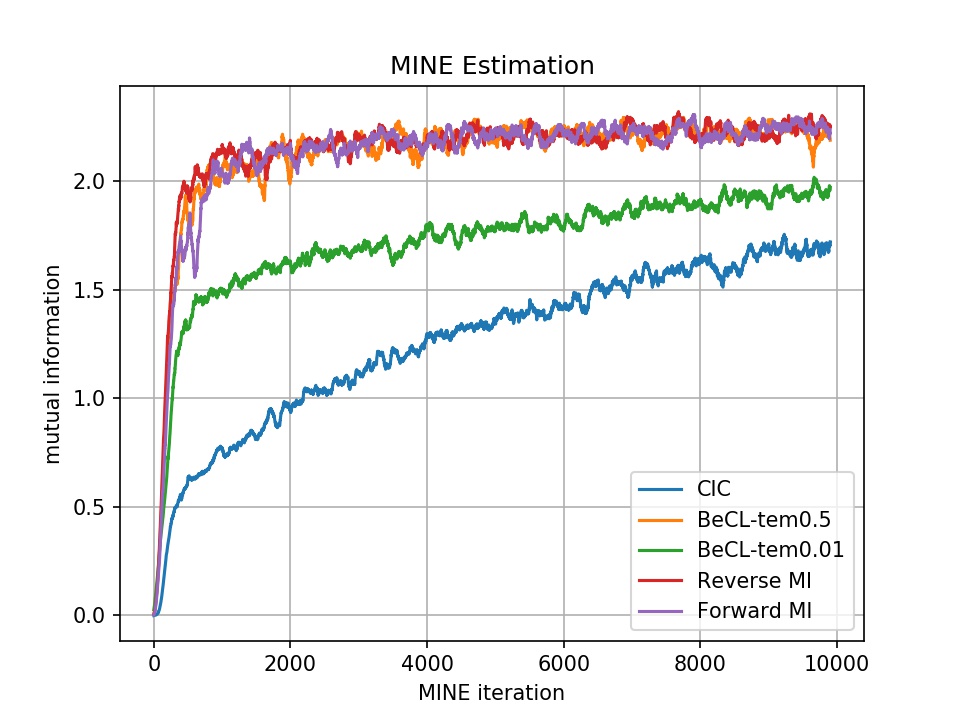}
}
\vspace{-1.5em}
\caption{We compare the mutual information (MI) estimation and the entropy estimation in maze. Specifically, we generate several trajectories for each learned skill and perform MI estimation using MINE \cite{mine-2018} and entropy estimation using the particle-based entropy estimator \cite{apt}. The result shows that CIC obtains much lower MI than other skill discovery methods but obtains the largest state entropy. BeCL can balance state coverage and empowerment with different temperature parameters, which leads to diverse skills and also better state coverage than previous MI-based algorithms.}
\vspace{-2em}
\label{fig:MI and Entropy Estimation}
\end{center}
\end{figure}

\newpage

\section{Additional Experiments in URLB}
\label{app:Additional Experiments in URLB}
\subsection{Hyperparameter}
We adopt the baselines of open source code implemented by URLB (\url{https://github.com/rll-research/url_benchmark}) and CIC (\url{https://github.com/rll-research/cic}). The hyperparameters of the baselines remain unchanged and are fixed in all tasks in the pretraining and finetuning stage. Table \ref{table:common_hyperparams} shows the hyperparameters of BeCL and DDPG. We also refer to the hyperparameters settings of baselines implemented in URLB.

\begin{table}[h]
\caption{\label{table:common_hyperparams} Hyper-parameters used for BeCL and DDPG.}
\centering
\begin{tabular}{lc}
\hline
\textbf{BeCL hyper-parameter}       & \textbf{Value} \\
\hline

\: Skill dim & 16 discrete \\
\: Temperature $\kappa$ & 0.5 \\
\: Skill sampling frequency (steps) & 50 \\
\: Contrastive encoder arch. $f(s)$ & $ \dim (S) \to 1024 \to 1024 \to 16 \to 1024 \to 16$ $\textrm{ReLU}$ (MLP)  \\
\hline
\textbf{DDPG hyper-parameter}       & \textbf{Value} \\
\hline
\: Replay buffer capacity & $10^6$ \\
\: Action repeat & $1$ \\
\: Seed frames & $4000$ \\
\: $n$-step returns & $3$ \\
\: Mini-batch size & $1024$  \\
\: Seed frames & $4000$ \\
\: Discount ($\gamma$) & $0.99$ \\
\: Optimizer & Adam \\
\: Learning rate & $10^{-4}$ \\
\: Agent update frequency & $2$ \\
\: Critic target EMA rate ($\tau_Q$) & $0.01$ \\
\: Features dim. & $1024$  \\
\: Hidden dim. & $1024$ \\
\: Exploration stddev clip & $0.3$ \\
\: Exploration stddev value & $0.2$ \\
\: Number pretraining frames &  $2\times 10^6$ \\
\: Number fineturning frames & $1 \times 10^5$ \\
\hline
\end{tabular}
\end{table}

\subsection{Description of Baselines in URLB}
\label{app:description_of_baselines_in_URLB}

A comparison of different intrinsic rewards and representations in unsupervised RL baselines in our experiments is shown in Table~\ref{table:baselines}. Specifically, knowledge-based baselines utilize a trainable encoder to predict dynamics $f(s_{t+1}|s_t,a_t)$ (e.g., ICM \cite{pathak2017curiosity}, Disagreement \cite{pathak2019disagreement}) or minimize the output error of $f(s_t,a_t)$ and a random network $\tilde{f}(s_t,a_t)$ (e.g., RND \cite{burda2018rnd}); data-based baselines maximize the entropy of collected data on different representations of state $f(s)$ with particle estimator; Competence-based baselines aim to learn latent skill $z$ by maximizing the MI between states and skills: $I(S;Z) = H(S) - H(S|Z) = H(Z) - H(Z|S)$. For example, APS \cite{liu2021aps} optimizes the forward form of $I(S;Z)$ as in DADS \cite{dads} but with successor features, and estimates $H(S)$ with the particle estimator as in APT; DIAYN \cite{diayn} optimizes the reverse form of $I(S;Z)$, and utilizes the non-negativity property of the KL divergence to compute the variational lower bound of $I(S;Z)$ through a trainable network $q$.
The main differences between the baselines in URLB are the design of intrinsic reward and its state representation.  More descriptions of the baselines can be found in URLB \cite{URLB}. Furthermore, BeCL is a competence-based method and trains skills with $I(S^{(1)};S^{(2)})$.

\begin{table}[h]
  \caption{BeCL and other unsupervised RL baselines.}
  \label{table:baselines}
  \centering
\resizebox{\columnwidth}{!}{
  \begin{tabular}{llll}
    \toprule

    Name     & Algo. Type     & Intrinsic Reward & Explicit max $H(s)$\\
    \midrule
    ICM~\cite{pathak2017curiosity} & Knowledge  & $ \| f(s_{t+1}|s_t,a_t) - s_{t+1} \|^2$ & No \\
    Disagreement~\cite{pathak2019disagreement}     & Knowledge & $\textrm{Var} \{ f_i(s_{t+1}|s_t,a_t) \} \quad i =1,\dots,N$   & No   \\
    RND~\cite{burda2018rnd}     & Knowledge       &  $\| f(s_t,a_t) - \tilde{f}(s_t, a_t) \|^2_2$ & No \\

    \midrule 
    APT~\cite{apt} & Data & $ \sum_{j\in \textrm{KNN}} \log \| f(s_t) - f(s_j) \| \quad f \in \textrm{random or ICM }$   & Yes \\
    ProtoRL~\cite{proto}     & Data & $\sum_{j\in \textrm{KNN}} \log \| f(s_t) - f(s_j) \| \quad f \in \textrm{prototypes}$     & Yes  \\
     CIC~\cite{cic}  & Data \tablefootnote{The newest NeurIPS version of CIC \url{https://openreview.net/forum?id=9HBbWAsZxFt} has two designs of intrinsic reward including the NCE term and KNN reward, which represent competence-base and data-based designs respectively. Since CIC obtains the best performance in URLB with KNN reward only and NCE is used to update representation, we use KNN reward as its intrinsic reward and consider it as a data-based method in this paper.}	& $ \sum_{j\in \textrm{KNN}} \log \| f(s_t,s'_t) - f(s_j,s'_j) \| \quad f \in \textrm{contrastive} $	& Yes \\     
   \midrule 
     SMM~\cite{lee2019smm}     & Competence       &  $\log p^*(s) -\log q_z(s)  - \log p(z) + \log d(z|s)$ & Yes \\
    DIAYN~\cite{diayn} & Competence  & $  \log q(z|s) + \log p(z)$   & No  \\
    APS~\cite{liu2021aps}    & Competence & $r^{\text{APT}}_{t}(s) + \log q(s|z)$   & Yes
    \\

    BeCL (\textbf{Ours})  &   Competence      & $\exp(f(s^{(1)}_t)^{\top}f(s^{(2)}_t)/\kappa ) / \sum_{s_j \sim S^{-}\bigcup s_t^{(2)}}\exp (f(s_j)^{\top}f(s^{(1)}_t)/\kappa $ & No \\
    \bottomrule
  \end{tabular}}
\end{table}

\begin{figure}[h]
\begin{center}
\centerline{
\includegraphics[width=0.8\columnwidth]{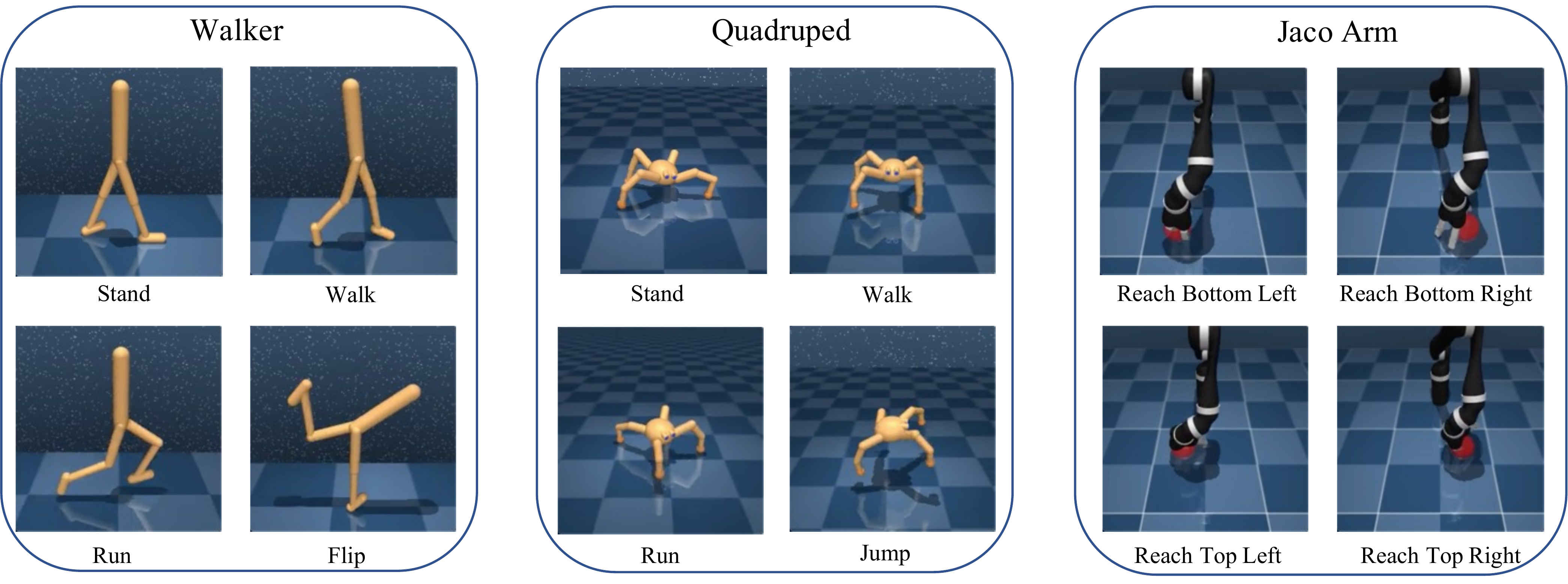}}
\caption{Introduction of domains and their downstream tasks in URLB \cite{URLB}. There are four different domains and each domain has four different downstream tasks. The environment is based on DMC \cite{dmc}. The episode lengths for the Walker and Quadruped domains are set to 1000, and the episode length for the Joco domain is set to 250, which results in the maximum episodic reward for the Walker and Quadruped domains being 1000, and for Jaco Arm being 250.}
\label{fig:intro_env_urlb}
\end{center}
\end{figure}

\subsection{Practical Implementation}\label{app-alg}

We evaluate the adaptation efficiency of BeCL following the pretraining and finetuning procedures in URLB \cite{URLB}. Specifically, in the pretraining stage, latent skill $z$ is changed and sampled from a discrete distribution $p(z)$ in every fixed step and the agent interacts with the environments based on $\pi_\theta(a|s,z)$. The encoder network $f(\cdot)$ will be updated after 4k steps of pretraining. We use 5 MLP to construct the encoder network and compute the optimization loss by Eq.~\eqref{eq:nce-tau}. The replay buffer has the same capacity on all baselines. We sample a mini-batch from the replay buffer $(s, s', a, z)$, and pick samples with the same skill as a positive pair, and consider those samples with different skills as negative pairs to compute contrastive loss and intrinsic reward. We then update the critic $Q_{\psi}$ by minimizing the Bellman residual, as
\begin{equation}
    \label{eq:critic}
    \mathcal{L}_Q(\psi,\mathcal{B})=\mathbb{E}_{(s_t, a_t, r_t, s_{t+1}, z_t)\sim\mathcal{B}}\left[\Big(Q_\psi(s_t,z_t,a_t) - r_t - \gamma Q_{\bar{\psi}}(s_{t+1},z_t,\pi_\theta(s_{t+1}, z_t)\Big)^2\right].
\end{equation}
where $\bar{\psi}$ is an exponential moving average (EMA) of the critic weights $\psi$, and $r_t$ is an intrinsic or extrinsic reward depending on the training stages. We train the actor $\pi_\theta(s_t,z_t)$ by maximizing the expected returns, as
\begin{equation}
    \label{eq:actor}
    \mathcal{L}_\pi(\theta, \mathcal{B})=-\mathbb{E}_{(s_t,z_t) \sim\mathcal{B}}\left[Q_\psi(s_t,z_t,\pi_\theta(s_t,z_t))\right].
\end{equation}
In the finetuning stage, a skill is randomly sampled and keep fixed in all steps. The actor and critic are updated by extrinsic reward after first 4000 steps. We give algorithmic descriptions of the pretraining and finetuning stages in Algorithm~\ref{app:pretrain_algo} and Algorithm~\ref{app:finetune_algo}, respectively. In our experiment, pretraining one seed of BeCL for 2M steps takes about 18 hours while fine-tuning to downstream tasks for 100k steps takes about 30 minutes with a A100 GPU.

\newpage
\begin{algorithm}[h]
\caption{BeCL: Unsupervised pretraining}\label{app:pretrain_algo}
\begin{algorithmic}
  \STATE {\bfseries Input:} number of pretraining frames $N_{PT}$, skill dimension $|z|$, batch size $N$,  and skill sampling frequency $N_{update}$.
  \STATE \textbf{Initialize} the environment, random actor $\pi_\theta(a | s,z)$, critic $Q_\psi(s,z, a)$, contrastive encoder $f$, and replay buffer $\mathcal{B}$
  \FOR{$t=1$ {\bfseries to} $N_{PT}$} 

    \STATE Randomly choose $z_t$ \text{from category distribution} $p(z) \in \RR^{|z|}$ every $N_{update}$ steps.
    
    \STATE Interact with environment : $\tau_{z_t}, \tau_{z_t}' \sim \pi_\theta(a | s, z_t), ~p(s'|s, a)$.
    \STATE Store $\tau_{z_t},\tau_{z_t}'$ into buffer $\mathcal{B}$.
    \IF{$t \geq 4,000$}
    \STATE Sample a batch from $\mathcal{B}$ : $\{(\mathbf{a}_i^{(1)}, \mathbf{s}_i^{(1)}, \mathbf{s}_i^{\prime(1)}, \mathbf{z}_i),(\mathbf{a}_i^{(2)},\mathbf{s}_i^{(2)}, \mathbf{s}^{\prime(2)}_i, \mathbf{z}_i)\}_{i=1}^{N / 2} \sim \{\tau_Z\}$.
    \STATE Update the contrastive encoder $f$ using contrastive loss in Eq.~\eqref{eq:nce-tau}.
    \STATE Compute the intrinsic reward $r^{\text{int}}$ with Eq.~\eqref{eq:reward}.
    \STATE Update actor $\pi_\theta(a | s,z)$ and critic $Q_\psi(s,z, a)$ by Eq.~\eqref{eq:critic} and Eq.~\eqref{eq:actor} using intrinsic reward $r^{\text{int}}$.
    \ENDIF
  \ENDFOR
\end{algorithmic}
\end{algorithm}
\vspace{-1em}
\begin{algorithm}[h]
   \caption{BeCL: Finetuning with extrinsic rewards} \label{app:finetune_algo}
\begin{algorithmic}
  \STATE {\bfseries Input: }actor $\pi_\theta(a | s,z^\star)$ and critic $Q_\psi([s,z^\star], a)$ with weights from pretraining phase, randomly sampled $z^\star$ from $p(z)$, and number of finetuning frames $N_{FT}$ batch size $N$. Initialized environment and an empty replay buffer $\mathcal{D}$.
  % \STATE Initialize environment $p(s^{\prime}| s, a)$,  and skill $z_t$
  \FOR{$t=1$ {\bfseries to} $N_{FT}$} 

    \STATE Choose the action by $a_t \sim \pi_\theta(a | s_t,z^\star)$.
    \STATE Interact with environment to obtain $s_{t+1}, r_t$ with extrinsic reward from downstream task.
    \STATE Store $(s_{t}, a_t, s_{t+1}, r_{t}, z^\star)$ into buffer $\mathcal{D}$.
    \IF{$t \geq 4,000$}
    \STATE Sample a batch $\{(\mathbf{a}^{(i)}, \mathbf{s}^{(i)}, \mathbf{s}^{\prime(i)}, \mathbf{r}^{(i)}, \mathbf{z}^{(i)})\}_{i=1}^{N}$ from the replay buffer $\mathcal{D}$.
    \STATE Update actor $\pi_\theta(a | s,z^\star)$ and critic $Q_\psi([s,z^\star], a)$ using extrinsic reward $r$ in Eq.~\eqref{eq:critic} and Eq.~\eqref{eq:actor}.
    \ENDIF
  \ENDFOR
\end{algorithmic}
\end{algorithm}

\subsection{Visualization of Behaviors in Competence-based Methods}
\label{app-vis-skill-dmc}

\begin{figure*}[h!]
\begin{center}
\centerline{
\includegraphics[width=0.8\columnwidth]{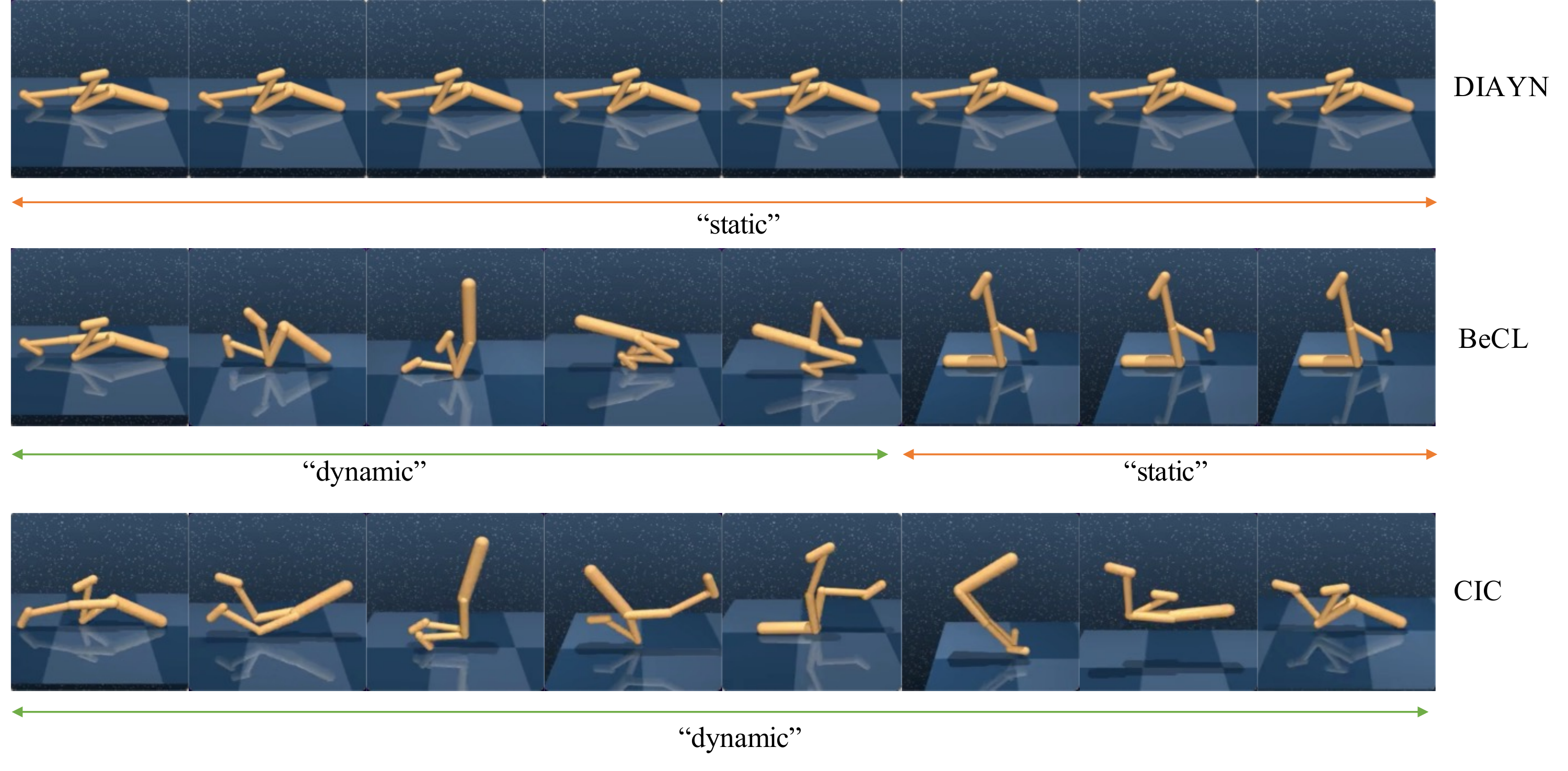}}
% \vspace{-1em}
\caption{Qualitative visualization for the behavior of different competence-based algorithms on \textit{Walker} domain from URLB. From left to right, the figures show the snapshots of behaviors from the three algorithms within an episode. We find that DIAYN polices produce constantly static and non-trivial poses, while CIC policies can produce dynamic and trivial behaviors, which is consistent with the observations of previous work \cite{cic, URLB}. In contrast, BeCL combines dynamic and static behaviors during an episode, which can also be use for good adaptation in downstream tasks.}
\label{fig:comparison_dmc_behavior}
\end{center}
\vspace{-1em}
\end{figure*}

\newpage
\subsection{Numerical Results}
We represent the individual normalized return of different methods in state-based URLB after 2M steps of pretraining and 100k steps of finetuning, as shown in Table~\ref{table:numerical_result}. In the \textit{Quadruped} domain, BeCL obtains state-of-the-art performance in downstream tasks. In the \textit{Walker} and \textit{Jaco} domains, BeCL shows competitive performance against the leading baselines.

\begin{table}[h]
\centering
\caption{Results of BeCL and other baselines on state-based URLB. All baselines are pretrained for 2M steps with only intrinsic rewards in each domain, and then finetuned to 100K steps in each downstream task by giving the extrinsic rewards. All baselines are run for 10 seeds per task, and the code and hyperparameters are given in URLB \cite{URLB}. The highest performing scores are highlighted.}
\vspace{0.5em}
\label{table:numerical_result}
\resizebox{\columnwidth}{!}{
\begin{tabular}{ccccccccccccc}
\toprule
% \multicolumn{13}{c}{Pretrainining for 2M frames in URLB}                         \\ \hline
Domain                     & \multicolumn{1}{c|}{Task}               & \multicolumn{1}{c|}{DDPG}   & ICM        & Disagreement & \multicolumn{1}{c|}{RND}        & APT        & \multicolumn{1}{c|}{ProtoRL}    & SMM        & DIAYN      & APS        & \multicolumn{1}{c|}{CIC}        & BeCL       \\ \hline
\multirow{4}{*}{Walker}    & \multicolumn{1}{c|}{Flip}               & \multicolumn{1}{c|}{538±27} & 390$\pm$10 & 332$\pm$7    & \multicolumn{1}{c|}{506$\pm$29} & 606$\pm$30 & \multicolumn{1}{c|}{549$\pm$21} & 500$\pm$28 & 361$\pm$10 & 448$\pm$36 & \multicolumn{1}{c|}{\textbf{641$\pm$26}} & \textbf{611$\pm$18} \\
                           & \multicolumn{1}{c|}{Run}                & \multicolumn{1}{c|}{325±25} & 267$\pm$23 & 243$\pm$14   & \multicolumn{1}{c|}{403$\pm$16} & 384$\pm$31 & \multicolumn{1}{c|}{370$\pm$22} & 395$\pm$18 & 184$\pm$23 & 176$\pm$18 & \multicolumn{1}{c|}{\textbf{450$\pm$19}} & 387$\pm$22 \\
                           & \multicolumn{1}{c|}{Stand}              & \multicolumn{1}{c|}{899±23} & 836$\pm$34 & 760$\pm$24   & \multicolumn{1}{c|}{901$\pm$19} & 921$\pm$15  & \multicolumn{1}{c|}{896$\pm$20} & 886$\pm$18 & 789$\pm$48 & 702$\pm$67 & \multicolumn{1}{c|}{\textbf{959$\pm$2}}  & \textbf{952$\pm$2}  \\
                           & \multicolumn{1}{c|}{Walk}               & \multicolumn{1}{c|}{748±47} & 696$\pm$46 & 606$\pm$51   & \multicolumn{1}{c|}{783$\pm$35} & 784$\pm$52 & \multicolumn{1}{c|}{836$\pm$25} & 792$\pm$42 & 450$\pm$37 & 547$\pm$38 & \multicolumn{1}{c|}{\textbf{903$\pm$21}} & \textbf{883$\pm$34} \\ \hline
\multirow{4}{*}{Quadruped} & \multicolumn{1}{c|}{Jump}               & \multicolumn{1}{c|}{236±48} & 205$\pm$47 & 510$\pm$28   & \multicolumn{1}{c|}{626$\pm$23} & 416$\pm$54 & \multicolumn{1}{c|}{573$\pm$40} & 167$\pm$30 & 498$\pm$45 & 389$\pm$72 & \multicolumn{1}{c|}{565$\pm$44} & \textbf{727$\pm$15} \\
                           & \multicolumn{1}{c|}{Run}                & \multicolumn{1}{c|}{157±31} & 125$\pm$32 & 357$\pm$24   & \multicolumn{1}{c|}{439$\pm$7}  & 303$\pm$30 & \multicolumn{1}{c|}{324$\pm$26} & 142$\pm$28 & 347$\pm$47 & 201$\pm$40 & \multicolumn{1}{c|}{445$\pm$36} & \textbf{535$\pm$13} \\
                           & \multicolumn{1}{c|}{Stand}              & \multicolumn{1}{c|}{392±73} & 260$\pm$45 & 579$\pm$64   & \multicolumn{1}{c|}{\textbf{839$\pm$25}} & 582$\pm$67 & \multicolumn{1}{c|}{625$\pm$76} & 266$\pm$48 & 718$\pm$81 & 435$\pm$68 & \multicolumn{1}{c|}{700$\pm$55} & \textbf{875$\pm$33} \\
                           & \multicolumn{1}{c|}{Walk}               & \multicolumn{1}{c|}{229±57} & 153$\pm$42 & 386$\pm$51   & \multicolumn{1}{c|}{517$\pm$41} & 582$\pm$67 & \multicolumn{1}{c|}{494$\pm$64} & 154$\pm$36 & 506$\pm$66 & 385$\pm$76 & \multicolumn{1}{c|}{621$\pm$69} & \textbf{743$\pm$68} \\ \hline
\multirow{4}{*}{Jaco}      & \multicolumn{1}{c|}{Reach bottom left}  & \multicolumn{1}{c|}{72±22}  & 88$\pm$14  & 117$\pm$9    & \multicolumn{1}{c|}{102$\pm$9}  & 143$\pm$12 & \multicolumn{1}{c|}{118$\pm$7}  & 45$\pm$7   & 20$\pm$5   & 84$\pm$5   & \multicolumn{1}{c|}{\textbf{154$\pm$6}}  & \textbf{148$\pm$13} \\
                           & \multicolumn{1}{c|}{Reach bottom right} & \multicolumn{1}{c|}{117±18} & 99$\pm$8   & 122$\pm$5    & \multicolumn{1}{c|}{110$\pm$7}  & \textbf{138$\pm$15} & \multicolumn{1}{c|}{\textbf{138$\pm$8}}  & 60$\pm$4   & 17$\pm$5   & 94$\pm$8   & \multicolumn{1}{c|}{\textbf{149$\pm$4}} & \textbf{139$\pm$14} \\
                           & \multicolumn{1}{c|}{Reach top left}     & \multicolumn{1}{c|}{116±22} & 80$\pm$13  & 121$\pm$14   & \multicolumn{1}{c|}{88$\pm$13}  & 137$\pm$20 & \multicolumn{1}{c|}{134$\pm$7}  & 39$\pm$5   & 12$\pm$5   & 74$\pm$10  & \multicolumn{1}{c|}{\textbf{149$\pm$10}} & 125$\pm$10 \\
                           & \multicolumn{1}{c|}{Reach top right}    & \multicolumn{1}{c|}{94±18}  & 106$\pm$14 & 128$\pm$11   & \multicolumn{1}{c|}{99$\pm$5}   & \textbf{170$\pm$7} & \multicolumn{1}{c|}{140$\pm$9}  & 32$\pm$4   & 21$\pm$3   & 83$\pm$11  & \multicolumn{1}{c|}{163$\pm$9}  & 126$\pm$10 \\ \hline           
\end{tabular}}
\end{table}

\subsection{Evaluation of Different Skills in Finetuning}

\begin{figure}[h!]
\begin{center}
\centerline{
\includegraphics[width=0.9\columnwidth]{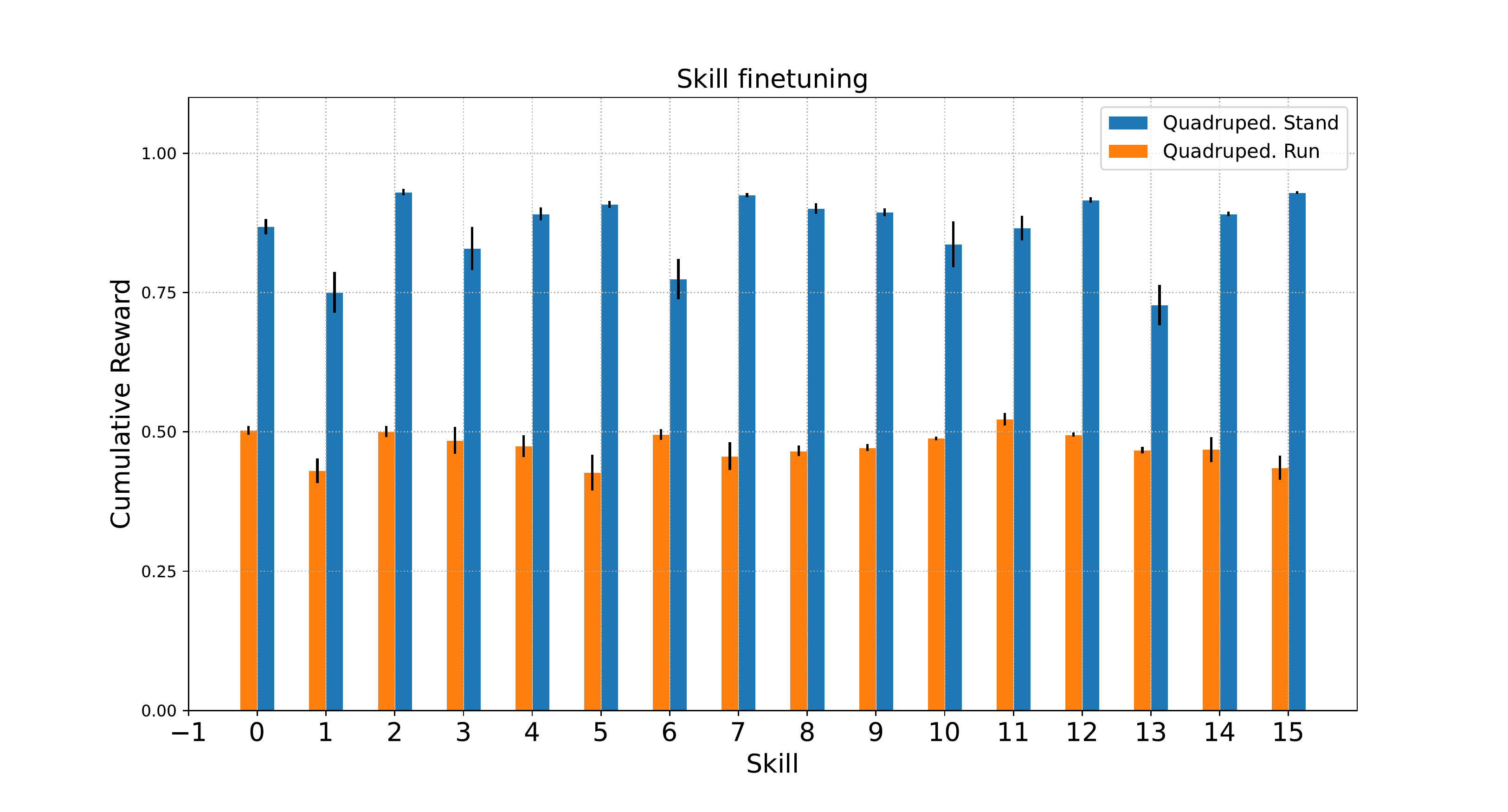}}
\caption{The adaptation efficiency of difference skills in the finetuning stage. The skill is a one-hot vector sampled from a 16-dimensional discrete distribution in the pretraining stage. Then we sample each skill for finetuning in \textit{(Quadruped, Run)} and \textit{(Quadruped, Stand)} tasks. We find that some skills do not always obtain a well generalization performance in downstream tasks (e.g. skill 1,6,13 in stand task). In BeCL, we uniformly choosing skill in the finetuning stage to evaluate the average adaptation performance among skills, although it would be better to choosing the best skills through a grid search like CIC \cite{cic}. We believe more matching meta-controller or other effective finetuning methods should be considered in future works.}
\label{fig:ablation_finetune_skills}
\end{center}
\end{figure}

\newpage
\subsection{The Impact of Skill Dimension in Adaptation}
\begin{figure}[h]
\begin{center}

\centerline{
\includegraphics[width=0.78\columnwidth]{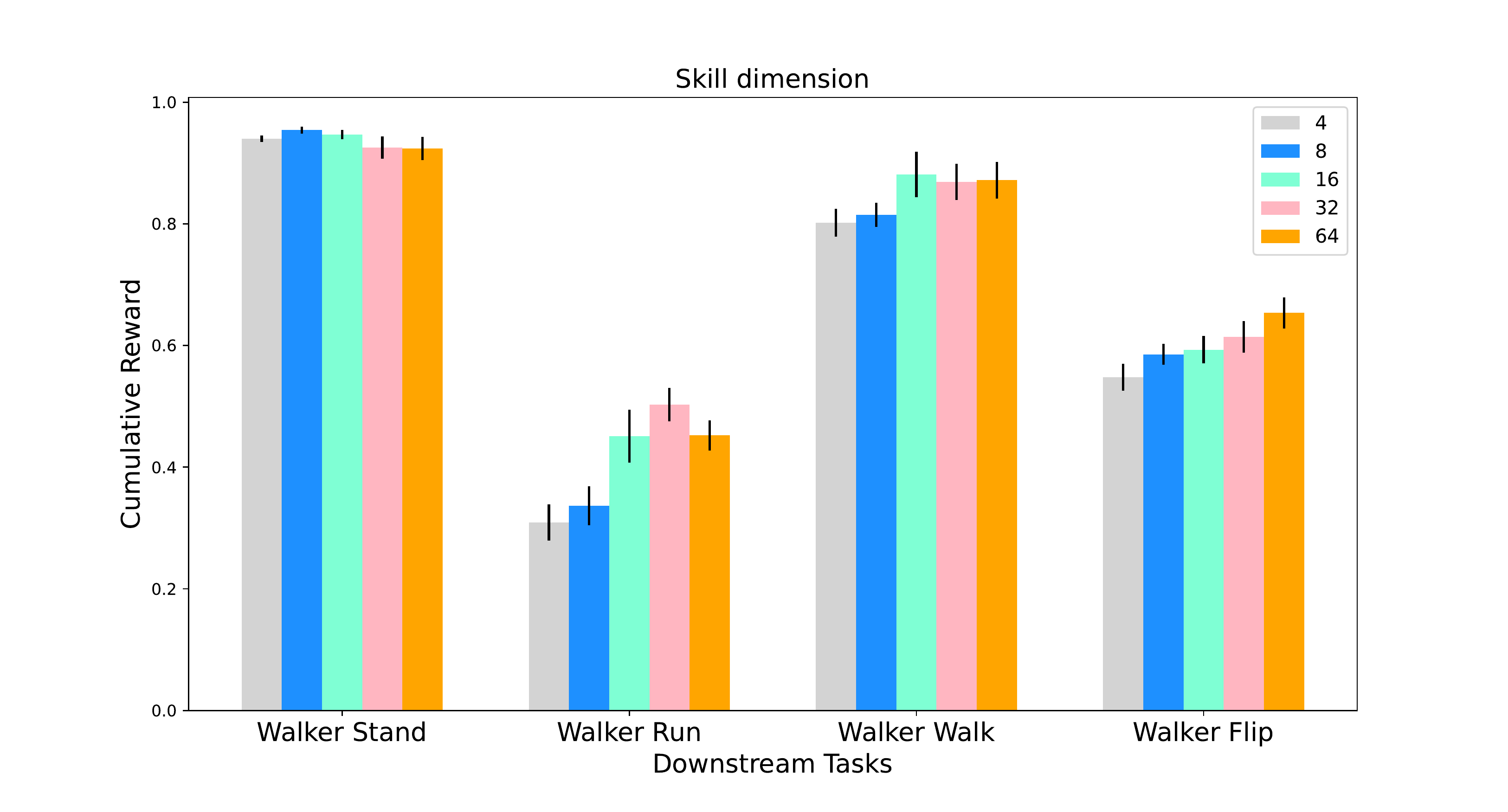}}
\caption{Illustration of the adaptation efficiency of different skill dimensions in URLB. We pretrain the policies with 4, 8, 16, 32, 64-dimensional discrete skills respectively. We compare their average normalize reward after 100K steps of finetuning. The results show that increasing the skill dimension in BeCL can benefit the adaptation efficiency in some hard downstream tasks (e.g. \textit{(Walker, Run)} and \textit{(Walker, Flip)} that need more complicated skills). Meanwhile, we find BeCL performs similar with different skill dimensions in easy downstream tasks like \textit{(Walker, Stand)}, where a 4-dimensional skill space can also perform well. In all URLB tasks, we use a 16-dimensional discrete skill distribution as in DIAYN \cite{diayn} for a fair comparison.}
\label{fig:ablation_study_urlb}
\end{center}
\end{figure}

\subsection{The Initial Results of Pixels-based URLB}

\begin{figure}[h!]
\begin{center}
\centerline{
\includegraphics[width=1.\columnwidth]{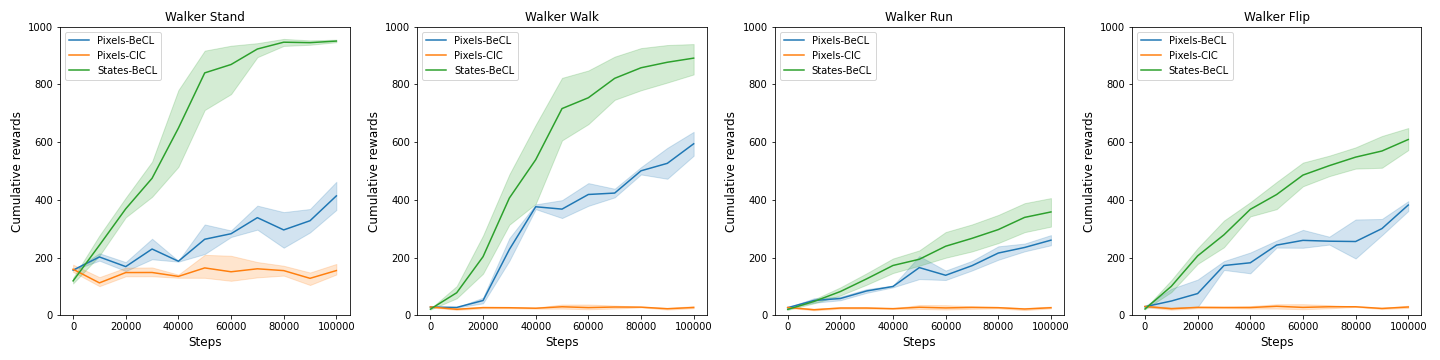}}
\caption{we provide the initial results of the performances of BeCL in pixels-based walker in URLB. To extend BeCL to image-based tasks, we add PIEG \cite{pieg} encoder to encode the observation ahead our policy and reward in the pretraining and finetuning stage. The result shows that the performance of BeCL in image-based tasks outperforms image-based CIC \cite{cic} while still underperforms that in state-based tasks.}
\label{fig:pixels-based walker}
\end{center}
\end{figure}

%%%%%%%%%%%%%%%%%%%%%%%%%%%%%%%%%%%%%%%%%%%%%%%%%%%%%%%%%%%%%%%%%%%%%%%%%%%%%%%
%%%%%%%%%%%%%%%%%%%%%%%%%%%%%%%%%%%%%%%%%%%%%%%%%%%%%%%%%%%%%%%%%%%%%%%%%%%%%%%

\end{document}